\documentclass[letterpaper]{article} 
\usepackage[preprint]{aaai2027}  
\usepackage[hyphens]{url}  
\usepackage{graphicx} 
\usepackage{natbib}  
\usepackage{caption} 
\usepackage{algorithm}
\usepackage{algorithmic}
\usepackage{amsthm}
\usepackage{amssymb}
\usepackage{newfloat}
\usepackage{listings}
\DeclareCaptionStyle{ruled}{labelfont=normalfont,labelsep=colon,strut=off} 
\floatstyle{ruled}
\newfloat{listing}{tb}{lst}{}
\floatname{listing}{Listing}

\usepackage{booktabs}

\usepackage{amsmath}
\newtheorem{assumption}{Assumption}
\newtheorem{definition}{Definition}
\newtheorem{proposition}{Proposition}
\newtheorem{example}{Example}
\newtheorem{theorem}{Theorem}
\newtheorem{corollary}{Corollary}

\usepackage{booktabs}
\usepackage{adjustbox}
\usepackage{caption}
\title{AVCG: A Generalized Variational Framework for Counterfactual Generation under Hypothesis Distributions}
\author{
    Jamie Duell\textsuperscript{\rm 1,2},
    Alejandro Jimenez Rodriguez\textsuperscript{\rm 1,2},
    Mahault Albarracin\textsuperscript{\rm 1,3,4,5}\\
}
\affiliations{
    \textsuperscript{\rm 1}School of Computing and Digital Technologies,
    Sheffield Hallam University\\
    
    \textsuperscript{\rm 2}Centre of Excellence in AI and Robotics (CEAIR),
    Sheffield Hallam University\\
    
    \textsuperscript{\rm 3}Laboratoire d’analyse cognitive de l’information (LANCI),
    Université du Québec à Montréal\\
    
    \textsuperscript{\rm 4}Institut Santé et société (ISS),
    Université du Québec à Montréal\\
    
    \textsuperscript{\rm 5}Institut de recherches et d’études féministes (IREF),
    Université du Québec à Montréal

    j.duell@shu.ac.uk
}

\begin{document}

\maketitle

\begin{abstract} 
Counterfactual explanations formalize ``what-if" scenarios by identifying modifications to an input instance that obtain a desired alternative prediction. Traditionally, whether generated via instance-specific optimization or amortized single-pass models, these approaches rely on a single, deterministic point-estimate predictor. However, this ignores predictive uncertainty and hypothesis variability, leading to brittle explanations that frequently become invalid if the underlying model is retrained or updated. To address this fragility, we propose the \emph{Amortized Variational Counterfactual Generator} (AVCG), a generalized optimization framework that formulates counterfactual generation as optimization over an arbitrary distribution of plausible predictive hypotheses rather than a single deterministic predictor. This formulation naturally accommodates Bayesian posteriors, Rashomon-restricted hypothesis spaces, and other uncertainty representations within a unified optimization framework. Evaluation across multiple benchmark datasets demonstrates that the AVCG framework produces counterfactual explanations that remain highly valid under predictive uncertainty and model changes, while maintaining competitive plausibility and single-pass runtime performance. 
\end{abstract}


\section{Introduction}

The deployment of machine learning models in high-stakes domains has necessitated the development of algorithmic recourse, typically provided through counterfactual explanations. However, a critical vulnerability in standard counterfactual generation is its reliance on a single, fixed predictive model. Explanations optimized against a deterministic decision boundary are often notoriously brittle; they frequently fail to remain valid if the model is retrained with a different random seed, subjected to minor data shifts, or if there is inherent epistemic uncertainty in the parameter space \cite{10.24963/ijcai.2024/894, sokol2025needcounterfactualexplainabilityprincipled}.

To provide reliable recourse, counterfactual generation must transition from targeting a single point-estimate predictor to optimizing over distributions of plausible hypotheses. Such hypothesis distributions arise naturally in several settings. Bayesian predictive models represent predictive uncertainty through posterior inference, as expressed through Bayesian Neural Networks \cite{10.5555/3045118.3045290} and Monte Carlo Dropout \cite{10.5555/3045390.3045502}, while model multiplicity motivates restricting attention to empirically near-optimal hypotheses, such as the \emph{Rashomon set}. Many existing uncertainty-aware counterfactual approaches, however, are typically tied to a specific uncertainty representation and frequently rely on computationally expensive post-hoc optimization at inference time. Notably, existing approaches remain specialized to particular uncertainty representations e.g., \cite{Schut2021GeneratingIC, Duell_2024} rather than providing a unified optimization framework capable of operating over arbitrary hypothesis distributions. For example, while Bayesian approaches optimize robustness in expectation under a posterior distribution, they do not directly support optimization over constrained hypothesis spaces, such as the Rashomon set \cite{hsu2026the}, where the objective is robustness across empirically near-optimal models rather than posterior expectation alone. Recent work has shown that Monte Carlo Dropout can be used to construct empirical approximations of such Rashomon sets by applying dropout to pretrained weights \cite{DBLP:conf/iclr/HsuLH024}.

To address these limitations, we propose the \emph{Amortized Variational Counterfactual Generator (AVCG)}, a generalized optimization framework for counterfactual generation over arbitrary hypothesis distributions. Unlike existing uncertainty-aware approaches \cite{Duell_2024, batten2025uncertaintyaware, Schut2021GeneratingIC} that predominantly generate explanations through iterative post-hoc optimization at inference time, AVCG learns a single amortized generator that produces counterfactuals in one forward pass while optimizing expected validity with respect to an arbitrary hypothesis distribution. Thus, the proposed AVCG framework generalizes the optimization objective beyond deterministic or Bayesian predictors and retains the computational efficiency required for practical deployment. Thus, to summarise AVCG shifts counterfactual generation from optimizing against a single decision boundary to optimizing against an entire distribution of plausible decision boundaries. This abstraction decouples the optimization objective from the mechanism used to represent uncertainty, allowing new uncertainty models or constrained hypothesis spaces to be incorporated without redesigning the counterfactual generation framework.

To this end, the contributions of this work can be summarized as follows:
\begin{enumerate}
    \item We formulate counterfactual generation as optimization over arbitrary hypothesis distributions, introducing AVCG, a generalized variational framework independent of any specific uncertainty representation.  
    \item We derive a generalized variational lower bound whose predictive objective is taken with respect to arbitrary hypothesis distributions, subsuming existing Bayesian formulations as special cases.
    \item We demonstrate the generality of AVCG by instantiating two fundamentally different robustness paradigms: Bayesian uncertainty and Rashomon model multiplicity, within a single optimization framework.
    \item We provide theoretical and empirical evaluations demonstrating robustness, validity and inference speed.
\end{enumerate}

\section{Preliminaries}

We formalize our framework by assuming a generalized predictive space that accounts for hypothesis variability. Given some parametric model $f_{\theta} :\mathbb{R}^d \rightarrow [0,1]$, let $P_\Phi(\theta)$ denote a generalized probability distribution over the model parameters $\theta$, defining the space of plausible hypotheses.
\begin{definition}[Hypothesis Distribution]
Let $\mathcal{H} = \{ f_{\theta} : \theta \in \Theta \}$ denote a hypothesis space of predictive models parameterized by $\theta$. A hypothesis distribution $P_{\Phi}(\theta)$, parameterized by $\Phi$, is a probability distribution over the parameter space $\Theta$ representing a practitioner's prior beliefs or constraints over plausible predictors.   
\end{definition}
\begin{definition}[Generalized Predictive Distribution]
Given a parametric model $f_{\theta}(\cdot) = p(y \mid \cdot, \theta)$ and a hypothesis distribution $P_\Phi(\theta)$, the predictive distribution for any input $\mathbf{x} \in \mathbb{R}^d$ is defined by marginalizing over the hypothesis space:
\begin{align*}
    p_\Phi(y \mid \mathbf{x}) = \int p(y \mid \mathbf{x}, \theta) \, P_\Phi(\theta) \, d\theta,
\end{align*}
which can equivalently be written as the expectation
\begin{align*}
    \mathbb{E}_{\theta \sim P_\Phi(\theta)}[f_{\theta}(\mathbf{x})].
\end{align*}
\end{definition}

A counterfactual instance is traditionally generated to flip the prediction of a single point-estimate model. A counterfactual instance in a model structured for a pointwise estimate is such that: 
\begin{definition}[Counterfactual Instance]
    Given an instance $\mathbf{x}  \in \mathbb{R}^d$, an instance $\mathbf{x}^\prime \in \mathbb{R}^d$, is said to be a counterfactual instance of $\mathbf{x}$, if for some counterfactual generator $g : \mathbb{R}^d \rightarrow \mathbb{R}^d$ s.t. $g(\mathbf{x}) = \mathbf{x}^\prime$, then: 
    \begin{align*}
        \underset{y}{\arg \max} \ f_{\theta}(\mathbf{x})  \neq  \underset{y}{\arg \max} \ f_{\theta}(\mathbf{x}^\prime). 
    \end{align*}
\end{definition}

To ensure robustness over hypothesis distributions, we define a robust counterfactual instance as follows:

\begin{definition}[Robust Counterfactual Instance]
    Given a factual instance $\mathbf{x} \in \mathbb{R}^d$ and a counterfactual generator $g : \mathbb{R}^d \rightarrow \mathbb{R}^d$ such that $g(\mathbf{x}) = \mathbf{x}^\prime$, the instance $\mathbf{x}^\prime$ is a robust counterfactual instance if: 
    \begin{align*}
        \underset{y}{\arg \max} \ p_\Phi(y \mid \mathbf{x}) \neq \underset{y}{\arg \max} \ p_\Phi(y \mid \mathbf{x}^\prime). 
    \end{align*}
\end{definition}

\section{Methodology}

We derive a novel, generalized variational objective for counterfactual generation that  accounts for predictive uncertainty by taking the expectation over $P_\Phi(\theta)$. We then instantiate this general framework into two distinct mechanisms to handle different robustness desiderata: the Bayesian posterior (AVCG-B) and the Rashomon-restricted posterior (AVCG-R).
\begin{figure*}[t]
    \centering
    \includegraphics[width=0.9\linewidth]{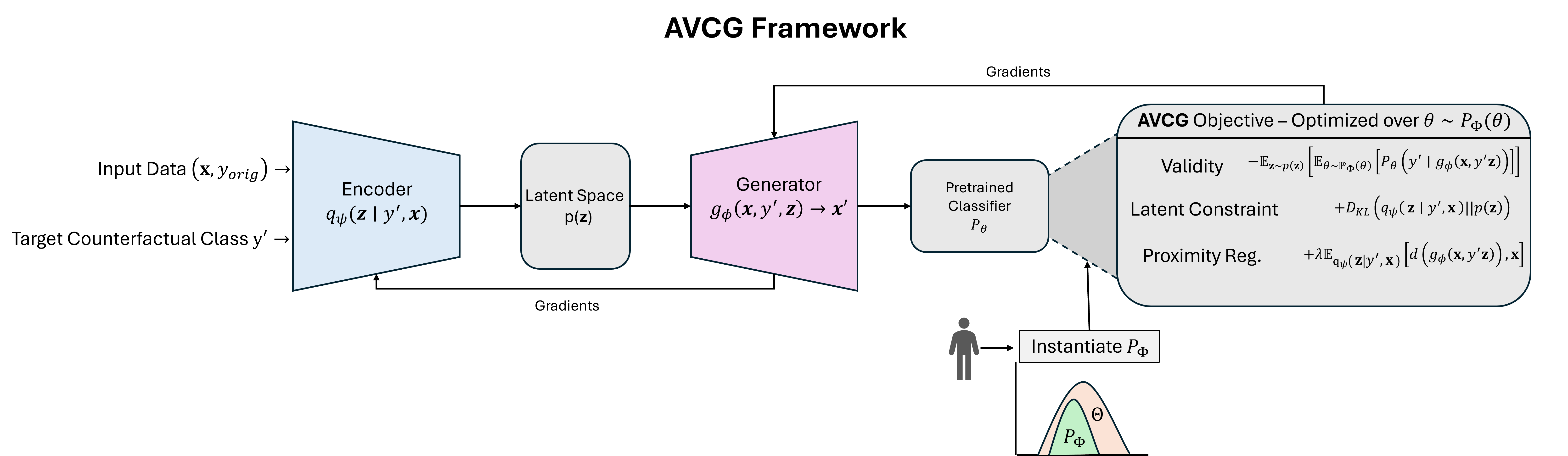}
    \caption{Overview of the AVCG Framework. In this work, $P_{\Phi}$ is instantiated as a Bayesian parametric posterior $P(\theta \mid D)$ or Rashomon set $P_{R}(\theta)$.}
    \label{fig:placeholder}
\end{figure*}

\subsection{The Generalized AVCG Framework}
To construct an amortized counterfactual generator, our goal is to maximize the predictive probability that the generated counterfactual samples achieve the desired label $y^\prime$ given the factual input $\mathbf{x}$. Operating within a latent-variable generative model, the true predictive distribution requires marginalizing over both the latent space $\mathbf{z}$ and the classifier's parameter space $\theta$. 

Assuming the prior over the latent distribution $p(\mathbf{z})$ is independent of the model's hypothesis space $P_\Phi(\theta)$, and that the generator $g_\phi$ is parameterized independently of $\theta$, we express the predictive objective as a double integral. We introduce a variational posterior $q_\psi(\mathbf{z} \mid y^\prime, \mathbf{x})$ to derive the lower bound:

\begin{align*}
     &\mathcal{L}(\mathbf{x},y^\prime) = \log \int \int P_{\theta}(y^\prime \mid g_{\phi}(\mathbf{x}, y^\prime, \mathbf{z}))p(\mathbf{z})P_\Phi(\theta) d\mathbf{z}d\theta \\
    &=  \log \int \int q_{\psi}(\mathbf{z} \mid y^\prime, \mathbf{x})\frac{P_{\theta}(y^\prime \mid g_{\phi}(\mathbf{x}, y^\prime, \mathbf{z}))p(\mathbf{z})P_\Phi(\theta)} {q_{\psi}(\mathbf{z} \mid y^\prime, \mathbf{x})}d\mathbf{z}d\theta \\
    &\geq   \int q_{\psi}(\mathbf{z} \mid y^\prime, \mathbf{x}) \bigg( \int P_\Phi(\theta)\log \ (P_{\theta}(y^\prime \mid g_{\phi}(\mathbf{x}, y^\prime, \mathbf{z}))) d\theta \bigg)d\mathbf{z} \\ &+ \int q_{\psi}(\mathbf{z} \mid y^\prime, \mathbf{x})\log \ \frac{p(\mathbf{z})}{q_{\psi}(\mathbf{z} \mid y^\prime, \mathbf{x})}d\mathbf{z}   \\ &\text{(Jensen's inequality and $q_{\psi}(\mathbf{z} \mid y^\prime, \mathbf{x}) \perp P_\Phi(\theta)$)}\\ 
    &= \mathbb{E}_{q_{\psi}(\mathbf{z} \mid y^\prime, \mathbf{x})}\bigg[\mathbb{E}_{\theta \sim P_\Phi(\theta)}\bigg[\log \ P_{\theta}(y^\prime \mid g_{\phi}(\mathbf{x}, y^\prime, \mathbf{z}))  \bigg] \bigg] \\ &+ \int q_{\psi}(\mathbf{z} \mid y^\prime, \mathbf{x}) \log \ \frac{p(\mathbf{z})}{q_{\psi}(\mathbf{z} \mid y^\prime, \mathbf{x})}d\mathbf{z}
    \\&= \mathbb{E}_{q_{\psi}(\mathbf{z} \mid y^\prime, \mathbf{x})}\bigg[\mathbb{E}_{\theta \sim P_\Phi(\theta)}\bigg[\log \ P_{\theta}(y^\prime \mid g_{\phi}(\mathbf{x}, y^\prime, \mathbf{z}))  \bigg] \bigg] \\ &+ \bigg( - D_{KL}(q_{\psi}(\mathbf{z} \mid y^\prime, \mathbf{x}) \mid\mid p(\mathbf{z})) \bigg). 
\end{align*}
Where $q_{\psi}$ is the encoder, $g_{\phi}$ is the generative model, and $P_{\theta}$ is the target parametric predictor.

\subsubsection{Proximity Regularization}
Finally, to align with standard counterfactual desiderata ensuring that interventions remain realistically actionable \cite{Mothilal_2020, sokol2025needcounterfactualexplainabilityprincipled}, we incorporate a proximity constraint using a distance function $d(\cdot)$: 
\begin{align*}
    \mathbb{E}_{q_{\psi}(\mathbf{z} \mid y^\prime, \mathbf{x})}\bigg[d(g_{\phi}(\mathbf{x}, y^\prime, \mathbf{z}), \mathbf{x}) \bigg].
\end{align*}
Let $D_{\text{train}} = \{ (\mathbf{x}_i, y^\prime_i) \}_{i=1}^N$, and 
\begin{align*}
    R_{\Phi}(\mathbf{x}^\prime) = \mathbb{E}_{q_{\psi}(\mathbf{z} \mid y^\prime, \mathbf{x})}\bigg[\mathbb{E}_{\theta \sim P_\Phi(\theta)}\bigg[\log \ P_{\theta}(y^\prime \mid g_{\phi}(\mathbf{x}, y^\prime, \mathbf{z}))  \bigg] \bigg]
\end{align*}
The final dataset-level objective for the generalized framework is:  
\begin{align*}
    {G}(\phi, \psi) = \sum^N_{i=1}\bigg[ - R_{\Phi}(\mathbf{x}^\prime) &+ D_{KL}(q_{\psi}(\mathbf{z} \mid y_i^\prime, \mathbf{x}_i) \mid\mid p(\mathbf{z})) \\ &+  \lambda \mathbb{E}_{q_{\psi}(\mathbf{z} \mid y_i^\prime, \mathbf{x})}\bigg[d(g_{\phi}(\mathbf{x}_i, y^\prime_i, \mathbf{z}), \mathbf{x}) \bigg]\bigg].
\end{align*}
whereby, optimal parameters $\phi^*$ and $\psi^*$ are obtained by minimising the objective, such that:
\begin{align*}
    \{\phi^*,\psi^* \} = \underset{\phi, \psi}{\arg \min }\ {G}(\mathbf{x}).  
\end{align*}

\subsection{AVCG-B: The Bayesian Instantiation}
When the practitioner's goal is to account for epistemic uncertainty and optimize for expected likelihood, we instantiate the general framework by setting $P_\Phi(\theta)$ to the full Bayesian posterior $P(\theta \mid D)$, such that our risk function is defined as: 
\begin{align*}
    R_{B}(\mathbf{x}^\prime) = \mathbb{E}_{q_{\psi}(\mathbf{z} \mid y^\prime, \mathbf{x})}\bigg[\mathbb{E}_{\theta \sim P(\theta\mid D ) }\bigg[\log \ P_{\theta}(y^\prime \mid g_{\phi}(\mathbf{x}, y^\prime, \mathbf{z}))  \bigg] \bigg],
\end{align*}

where we can substitute $R_{\Phi}(\mathbf{x}^\prime)$ with $R_{B}(\mathbf{x}^\prime)$  for the Bayesian posterior.

In practice, we approximate this posterior using stochastic weight masking via Monte Carlo Dropout \cite{10.5555/3045390.3045502}. During the training of AVCG, the generator optimizes against a dynamically sampled distribution of stochastic dropout masks. This forces the generated counterfactuals to achieve validity in expectation across the uncertainty of the network's weights, making it ideal for average-case generalization.

\subsection{AVCG-R: The Rashomon-Restricted Instantiation}
While the Bayesian posterior provides a probabilistic distribution of models, it does not explicitly enforce hard performance boundaries. In high-stakes domains, robustness across a constrained set of empirically near-optimal models may be desirable, namely the Rashomon set.

To target robustness over this empirically constrained hypothesis space, we instantiate $P_\Phi(\theta)$ as the Rashomon-restricted distribution $P_R(\theta)$. We define the empirical Rashomon set $\Theta_{R}$ as:
\begin{align*}
    \Theta_{R} = \{\theta : L(\theta) \leq L^* + \epsilon\},
\end{align*}
where $\epsilon$ specifies the maximum allowable deviation from the optimal validation loss $L^*$. The corresponding restricted posterior is:
\begin{align*}
    P_{R}(\theta) = \frac{P(\theta \mid D)\mathbf{1}[\theta \in \Theta_{R}]}{\int_{\Theta_R} P(\theta \mid D) \, d\theta}.
\end{align*}

To implement this reliably, AVCG-R replaces the stochastic dropout of the Bayesian instantiation with a set of \emph{deterministic, frozen weight masks}. Each mask is independently evaluated on a validation set, and only those satisfying $L(\theta_i) \leq L^* + \epsilon$ are retained. The generator then optimizes uniformly over this fixed hypothesis space. 

Modifying our loss function, we can instead integrate over the $P_{R}(\theta)$, such that our risk function conditioned on some $\epsilon$ is defined as: 
\begin{align*}
    R_{\epsilon}(\mathbf{x}^\prime) = \mathbb{E}_{q_{\psi}(\mathbf{z} \mid y^\prime, \mathbf{x})}\bigg[\mathbb{E}_{\theta \sim P_{R}(\theta) }\bigg[\log \ P_{\theta}(y^\prime \mid g_{\phi}(\mathbf{x}, y^\prime, \mathbf{z}))  \bigg] \bigg],
\end{align*}

where we can substitute $R_{\Phi}(\mathbf{x}^\prime)$ with $R_{\epsilon}(\mathbf{x}^\prime)$.

\subsection{Theoretical Analysis}
We begin by analysing the stability of the AVCG objective under divergence between hypothesis distributions.
\begin{theorem}[AVCG Stability]\label{thrm:stability}
Assume that there exists a constant $M>0$ such that
\[
\vert
\log P_{\theta}(y^\prime \mid g_{\phi}(\mathbf{x},y^\prime,\mathbf{z}))
\vert
\leq M
\]
for every $\theta$ in the supports of $P_\Phi^a$ and $P_\Phi^b$. Let
\begin{align*}
L_a(\phi,\psi) = R^a_{\Phi}(\mathbf{x}^\prime) &+ D_{KL}(q_{\psi}(\mathbf{z} \mid y_i^\prime, \mathbf{x}_i) \mid\mid p(\mathbf{z})) \\ &+  \lambda \mathbb{E}_{q_{\psi}(\mathbf{z} \mid y_i^\prime, \mathbf{x})}\bigg[d(g_{\phi}(\mathbf{x}_i, y^\prime_i, \mathbf{z}), \mathbf{x}) \bigg],
\end{align*}
and
\begin{align*}
L_b(\phi,\psi) = R^b_{\Phi}(\mathbf{x}^\prime) &+ D_{KL}(q_{\psi}(\mathbf{z} \mid y_i^\prime, \mathbf{x}_i) \mid\mid p(\mathbf{z})) \\ &+  \lambda \mathbb{E}_{q_{\psi}(\mathbf{z} \mid y_i^\prime, \mathbf{x})}\bigg[d(g_{\phi}(\mathbf{x}_i, y^\prime_i, \mathbf{z}), \mathbf{x}) \bigg],
\end{align*}
denote the AVCG objectives induced by the hypothesis distributions
$P_\Phi^a$
and
$P_\Phi^b$,
respectively, such that: 
\begin{align*}
     R^{(\cdot)}_{\Phi}(\mathbf{x}^\prime) = \mathbb{E}_{q_{\psi}(\mathbf{z} \mid y^\prime, \mathbf{x})}\bigg[\mathbb{E}_{\theta \sim P^{(\cdot)}_\Phi(\theta)}\bigg[\log \ P_{\theta}(y^\prime \mid g_{\phi}(\mathbf{x}, y^\prime, \mathbf{z}))  \bigg] \bigg]
\end{align*}
where ${(\cdot)}$ can be $a$ and $b$ respectively. Then, for fixed encoder and generator parameters $(\phi,\psi)$,
\[
\vert L_a(\phi,\psi)-L_b(\phi,\psi)\vert 
\le
\sqrt{2}M
\sqrt{
D_{KL}(P_\Phi^a\mid\mid P_\Phi^b)
}.
\]
\end{theorem}
This illustrates that small changes in the learned hypothesis distribution induce proportionally small changes in the optimization objective, providing stability under hypothesis class uncertainty.

\begin{proposition}
Let $f_\theta(\mathbf{x}^\prime) := P_\theta(y^\prime \mid g_{\phi}(\cdot))$. For any $0 < \gamma < 1$ and $f_{\theta}(\mathbf{x}^\prime) \in [0,1],$
\begin{align*}
    \mathbb{P}(f_{\theta}(\mathbf{x}^\prime) \geq \gamma) \geq 1- \frac{1-\mathbb{E}_{\theta \sim P_{\Phi}(\theta)}[f_{\theta}(\mathbf{x}^\prime)]}{1 -\gamma}.
\end{align*}
\end{proposition}

This bound shows that optimizing the generalized AVCG objective specifically under some hypothesis set $P_{\Phi}(\theta)$ satisfies $    f_{\theta}(\mathbf{x}^\prime) \geq \gamma$ with high probability.

\subsubsection{Rashomon Example}
Consider a newly trained model $\theta^*$, assuming the newly trained model belongs to the hypothesis set, in this case for the Rashomon set this can be expressed via the assumption that $\theta^* \sim P_{R}(\theta)$. Repeated training with stochastic initialization produces solutions with comparable empirical loss, which can be treated as approximate empirical samples from the Rashomon-restricted posterior $P_R(\theta)$, leading to Corollary \ref{central_corol}.
\begin{assumption}\label{var_bound}
The variance over a Rashomon set of models is bounded, such that: 
\begin{align*}
     \mathbb{V}_{\theta \sim P_{R}(\theta)}(f_{\theta}(\mathbf{x}^\prime)) \leq \sigma^2.
\end{align*}
\end{assumption}
\begin{corollary}\label{central_corol}
Let $\sigma^2 = \mathrm{Var}_{\theta\sim P_R} (f_\theta(\mathbf{x}^\prime))$.  Then for any $\beta>0$,
\begin{align*}
    \mathbb{P}_{\theta^* \sim P_R} \Big( f_{\theta^*}(\mathbf{x}^\prime) > \mathbb{E}_{\theta\sim P_R}[f_\theta(\mathbf{x}^\prime)] - \beta \Big) \ge 1-\frac{\sigma^2}{\beta^2}.
\end{align*}
\end{corollary}
\begin{table*}[t]
\centering
\caption{Amortized Counterfactual Evaluation on Adult Income and Breast Cancer (Mean $\pm$ std over 5 seeds). Methods are grouped by the evaluation environment: Strict Empirical Bound ($\epsilon=0.0$) and Relaxed Empirical Bound ($\epsilon=0.8$).}
\label{tab:results_adult_breast1}
\resizebox{\textwidth}{!}{
\begin{tabular}{l ccc ccccc c}
\toprule
\textbf{Method} & \textbf{Val} $\uparrow$ & \textbf{IM1} $\downarrow$ & \textbf{Imp} $\downarrow$ & \textbf{Div} $\uparrow$ & \textbf{CMV} $\uparrow$ & \textbf{NE} $\downarrow$ & \textbf{R$_{IC}$} $\downarrow$ & \textbf{RVR} $\uparrow$ & \textbf{Time (s)} $\downarrow$ \\
\midrule
\multicolumn{10}{c}{\textit{Dataset: Adult Income}} \\
\midrule
\multicolumn{10}{l}{\textbf{Environment: $\epsilon = 0.0$ (Strict Empirical Bound)}} \\
\midrule
Batten et al. & 0.910 $\pm$ 0.286 & 1.403 $\pm$ 1.722 & 11.639 $\pm$ 3.361 & 0.000 $\pm$ 0.000 & 0.749 $\pm$ 0.278 & 0.0054 $\pm$ 0.0073 & 0.0342 $\pm$ 0.1030 & 0.675 $\pm$ 0.388 & 0.0171 $\pm$ 0.1092 \\
QUCE          & 0.918 $\pm$ 0.274 & 0.963 $\pm$ 0.322 & 9.807 $\pm$ 2.055  & 0.000 $\pm$ 0.000 & 0.883 $\pm$ 0.252 & 0.0012 $\pm$ 0.0031 & 0.0054 $\pm$ 0.0156 & 0.923 $\pm$ 0.251 & 0.0102 $\pm$ 0.0740 \\
Schut et al.  & 0.982 $\pm$ 0.133 & 1.943 $\pm$ 2.378 & 12.033 $\pm$ 3.225 & 0.000 $\pm$ 0.000 & 0.951 $\pm$ 0.163 & 0.0036 $\pm$ 0.0052 & 0.0163 $\pm$ 0.1533 & 0.915 $\pm$ 0.185 & 0.0443 $\pm$ 0.1013 \\
\textbf{AVCG-B}        & \textbf{1.000} $\pm$ \textbf{0.000} & \textbf{0.769} $\pm$ \textbf{0.214} & 8.320 $\pm$ 0.047 & \textbf{0.030} $\pm$ \textbf{0.010} & \textbf{1.000} $\pm$ \textbf{0.000} & \textbf{0.0000} $\pm$ \textbf{0.0000} & \textbf{0.0000} $\pm$ \textbf{0.0000} & \textbf{1.000} $\pm$ \textbf{0.000} & \textbf{0.0011} $\pm$ \textbf{0.0001} \\
\textbf{AVCG-R ($\epsilon=0.0$)} & \textbf{1.000} $\pm$ \textbf{0.000} & 0.785 $\pm$ 0.178 & \textbf{8.303} $\pm$ \textbf{0.059} & 0.017 $\pm$ 0.006 & \textbf{1.000} $\pm$ \textbf{0.000} & \textbf{0.0000} $\pm$ \textbf{0.0000} & \textbf{0.0000} $\pm$ \textbf{0.0000} & \textbf{1.000} $\pm$ \textbf{0.000} & \textbf{0.0011} $\pm$ \textbf{0.0001} \\
\midrule
\multicolumn{10}{l}{\textbf{Environment: $\epsilon = 0.8$ (Relaxed Empirical Bound)}} \\
\midrule
Batten et al. & 0.912 $\pm$ 0.283 & 1.403 $\pm$ 1.722 & 11.639 $\pm$ 3.361 & 0.000 $\pm$ 0.000 & 0.749 $\pm$ 0.278 & 0.0057 $\pm$ 0.0090 & 0.0355 $\pm$ 0.0903 & 0.774 $\pm$ 0.197 & 0.0171 $\pm$ 0.1092 \\
QUCE          & 0.918 $\pm$ 0.274 & 0.963 $\pm$ 0.322 & 9.807 $\pm$ 2.055  & 0.000 $\pm$ 0.000 & 0.883 $\pm$ 0.252 & 0.0013 $\pm$ 0.0040 & 0.0047 $\pm$ 0.0144 & 0.919 $\pm$ 0.235 & 0.0102 $\pm$ 0.0740 \\
Schut et al.  & 0.982 $\pm$ 0.133 & 1.943 $\pm$ 2.378 & 12.033 $\pm$ 3.225 & 0.000 $\pm$ 0.000 & 0.951 $\pm$ 0.163 & 0.0037 $\pm$ 0.0055 & 0.0173 $\pm$ 0.1483 & 0.955 $\pm$ 0.131 & 0.0443 $\pm$ 0.1013 \\
\textbf{AVCG-B}        & \textbf{1.000} $\pm$ \textbf{0.000} & 0.769 $\pm$ 0.214 & 8.320 $\pm$ 0.047 & \textbf{0.030} $\pm$ \textbf{0.010} & \textbf{1.000} $\pm$ \textbf{0.000} & \textbf{0.0000} $\pm$ \textbf{0.0000} & \textbf{0.0000} $\pm$ \textbf{0.0000} & \textbf{1.000} $\pm$ \textbf{0.000} & \textbf{0.0011} $\pm$ \textbf{0.0001} \\
\textbf{AVCG-R ($\epsilon=0.8$)} & \textbf{1.000} $\pm$ \textbf{0.000} & \textbf{0.764} $\pm$ \textbf{0.209} & \textbf{8.308} $\pm$ \textbf{0.051} & 0.022 $\pm$ 0.007 & \textbf{1.000} $\pm$ \textbf{0.000} & \textbf{0.0000} $\pm$ \textbf{0.0000} & \textbf{0.0000} $\pm$ \textbf{0.0000} & \textbf{1.000} $\pm$ \textbf{0.000} & \textbf{0.0011} $\pm$ \textbf{0.0001} \\
\midrule
\multicolumn{10}{c}{\textit{Dataset: Breast Cancer}} \\
\midrule
\multicolumn{10}{l}{\textbf{Environment: $\epsilon = 0.0$ (Strict Empirical Bound)}} \\
\midrule
Batten et al. & 0.948 $\pm$ 0.222 & 1.425 $\pm$ 0.714 & 7.225 $\pm$ 1.204 & 0.000 $\pm$ 0.000 & 0.655 $\pm$ 0.336 & 0.0051 $\pm$ 0.0063 & 0.0142 $\pm$ 0.0623 & 0.870 $\pm$ 0.187 & 0.0111 $\pm$ 0.0062 \\
QUCE          & \textbf{1.000} $\pm$ \textbf{0.000} & 0.787 $\pm$ 0.520 & 6.173 $\pm$ 1.117 & 0.000 $\pm$ 0.000 & 0.992 $\pm$ 0.046 & 0.0012 $\pm$ 0.0028 & \textbf{0.0004} $\pm$ \textbf{0.0031} & 0.995 $\pm$ 0.033 & 0.0103 $\pm$ 0.0056 \\
Schut et al.  & \textbf{1.000} $\pm$ \textbf{0.000} & 1.163 $\pm$ 0.330 & 8.036 $\pm$ 1.371 & 0.000 $\pm$ 0.000 & 0.662 $\pm$ 0.313 & 0.0017 $\pm$ 0.0018 & 0.0153 $\pm$ 0.0542 & 0.996 $\pm$ 0.023 & 0.0862 $\pm$ 0.0480 \\
\textbf{AVCG-B}        & \textbf{1.000} $\pm$ \textbf{0.000} & \textbf{0.579} $\pm$ \textbf{0.153} & \textbf{5.317} $\pm$ \textbf{0.594} & \textbf{1.135} $\pm$ \textbf{0.171} & \textbf{1.000} $\pm$ \textbf{0.000} & \textbf{0.0000} $\pm$ \textbf{0.0000} & 0.0477 $\pm$ 0.0391 & \textbf{1.000} $\pm$ \textbf{0.000} & \textbf{0.0011} $\pm$ \textbf{0.0001} \\
\textbf{AVCG-R ($\epsilon=0.0$)} & \textbf{1.000} $\pm$ \textbf{0.000} & 0.658 $\pm$ 0.147 & 5.364 $\pm$ 0.544 & 1.078 $\pm$ 0.209 & \textbf{1.000} $\pm$ \textbf{0.000} & 0.0001 $\pm$ 0.0002 & 0.0443 $\pm$ 0.0459 & \textbf{1.000} $\pm$ \textbf{0.000} & \textbf{0.0010} $\pm$ \textbf{0.0001} \\
\midrule
\multicolumn{10}{l}{\textbf{Environment: $\epsilon = 0.8$ (Relaxed Empirical Bound)}} \\
\midrule
Batten et al. & 0.948 $\pm$ 0.222 & 1.425 $\pm$ 0.714 & 7.225 $\pm$ 1.204 & 0.000 $\pm$ 0.000 & 0.655 $\pm$ 0.336 & 0.0045 $\pm$ 0.0047 & 0.0138 $\pm$ 0.0522 & 0.855 $\pm$ 0.161 & 0.0111 $\pm$ 0.0062 \\
QUCE          & \textbf{1.000} $\pm$ \textbf{0.000} & 0.787 $\pm$ 0.520 & 6.173 $\pm$ 1.117 & 0.000 $\pm$ 0.000 & 0.992 $\pm$ 0.046 & 0.0010 $\pm$ 0.0022 & \textbf{0.0003} $\pm$ \textbf{0.0021} & 0.995 $\pm$ 0.020 & 0.0103 $\pm$ 0.0056 \\
Schut et al.  & \textbf{1.000} $\pm$ \textbf{0.000} & 1.163 $\pm$ 0.330 & 8.036 $\pm$ 1.371 & 0.000 $\pm$ 0.000 & 0.662 $\pm$ 0.313 & 0.0018 $\pm$ 0.0021 & 0.0246 $\pm$ 0.0967 & 0.992 $\pm$ 0.017 & 0.0862 $\pm$ 0.0480 \\
\textbf{AVCG-B}        & \textbf{1.000} $\pm$ \textbf{0.000} & \textbf{0.579} $\pm$ \textbf{0.153} & \textbf{5.317} $\pm$ \textbf{0.594} & \textbf{1.135} $\pm$ \textbf{0.171} & \textbf{1.000} $\pm$ \textbf{0.000} & \textbf{0.0000} $\pm$ \textbf{0.0000} & 0.0466 $\pm$ 0.0397 & \textbf{1.000} $\pm$ \textbf{0.000} & \textbf{0.0011} $\pm$ \textbf{0.0001} \\
\textbf{AVCG-R ($\epsilon=0.8$)} & \textbf{1.000} $\pm$ \textbf{0.000} & 0.639 $\pm$ 0.129 & 5.377 $\pm$ 0.584 & 1.112 $\pm$ 0.176 & \textbf{1.000} $\pm$ \textbf{0.000} & \textbf{0.0000} $\pm$ \textbf{0.0000} & 0.0455 $\pm$ 0.0361 & \textbf{1.000} $\pm$ \textbf{0.000} & \textbf{0.0011} $\pm$ \textbf{0.0001} \\
\bottomrule
\end{tabular}
}
\end{table*}

\begin{table*}[t]
\centering
\caption{Amortized Counterfactual Evaluation on Heart Disease and Spambase (Mean $\pm$ std over 5 seeds).}
\label{tab:results_heart_spam1}
\resizebox{\textwidth}{!}{
\begin{tabular}{l ccc ccccc c}
\toprule
\textbf{Method} & \textbf{Val} $\uparrow$ & \textbf{IM1} $\downarrow$ & \textbf{Imp} $\downarrow$ & \textbf{Div} $\uparrow$ & \textbf{CMV} $\uparrow$ & \textbf{NE} $\downarrow$ & \textbf{R$_{IC}$} $\downarrow$ & \textbf{RVR} $\uparrow$ & \textbf{Time (s)} $\downarrow$ \\
\midrule
\multicolumn{10}{c}{\textit{Dataset: Heart Disease}} \\
\midrule
\multicolumn{10}{l}{\textbf{Environment: $\epsilon = 0.0$ (Strict Empirical Bound)}} \\
\midrule
Batten et al. & 0.891 $\pm$ 0.312 & 1.085 $\pm$ 0.427 & 4.740 $\pm$ 0.459 & 0.000 $\pm$ 0.000 & 0.591 $\pm$ 0.329 & 0.0029 $\pm$ 0.0027 & 0.0292 $\pm$ 0.0853 & 0.712 $\pm$ 0.224 & 0.0126 $\pm$ 0.0060 \\
QUCE          & 0.985 $\pm$ 0.122 & 0.780 $\pm$ 0.498 & 3.844 $\pm$ 0.375 & 0.000 $\pm$ 0.000 & 0.934 $\pm$ 0.158 & 0.0012 $\pm$ 0.0012 & \textbf{0.0005} $\pm$ \textbf{0.0045} & 0.966 $\pm$ 0.096 & 0.0084 $\pm$ 0.0040 \\
Schut et al.  & \textbf{1.000} $\pm$ \textbf{0.000} & 0.802 $\pm$ 0.201 & 5.203 $\pm$ 0.518 & 0.000 $\pm$ 0.000 & 0.997 $\pm$ 0.030 & 0.0008 $\pm$ 0.0009 & 0.0113 $\pm$ 0.0492 & \textbf{1.000} $\pm$ \textbf{0.000} & 0.0701 $\pm$ 0.0250 \\
\textbf{AVCG-B}        & \textbf{1.000} $\pm$ \textbf{0.000} & \textbf{0.407} $\pm$ \textbf{0.163} & \textbf{3.616} $\pm$ \textbf{0.123} & 0.670 $\pm$ 0.127 & \textbf{1.000} $\pm$ \textbf{0.000} & \textbf{0.0002} $\pm$ \textbf{0.0002} & 0.0211 $\pm$ 0.0180 & \textbf{1.000} $\pm$ \textbf{0.000} & 0.0011 $\pm$ 0.0001 \\
\textbf{AVCG-R ($\epsilon=0.0$)} & \textbf{1.000} $\pm$ \textbf{0.000} & 0.477 $\pm$ 0.174 & 3.661 $\pm$ 0.145 & \textbf{0.695} $\pm$ \textbf{0.130} & \textbf{1.000} $\pm$ \textbf{0.000} & \textbf{0.0002} $\pm$ \textbf{0.0003} & 0.0235 $\pm$ 0.0208 & \textbf{1.000} $\pm$ \textbf{0.000} & \textbf{0.0010} $\pm$ \textbf{0.0001} \\
\midrule
\multicolumn{10}{l}{\textbf{Environment: $\epsilon = 0.8$ (Relaxed Empirical Bound)}} \\
\midrule
Batten et al. & 0.884 $\pm$ 0.321 & 1.085 $\pm$ 0.427 & 4.740 $\pm$ 0.459 & 0.000 $\pm$ 0.000 & 0.591 $\pm$ 0.329 & 0.0032 $\pm$ 0.0037 & 0.0209 $\pm$ 0.0534 & 0.707 $\pm$ 0.150 & 0.0126 $\pm$ 0.0060 \\
QUCE          & 0.993 $\pm$ 0.086 & 0.780 $\pm$ 0.498 & 3.844 $\pm$ 0.375 & 0.000 $\pm$ 0.000 & 0.934 $\pm$ 0.158 & 0.0011 $\pm$ 0.0017 & \textbf{0.0002} $\pm$ \textbf{0.0010} & 0.946 $\pm$ 0.104 & 0.0084 $\pm$ 0.0040 \\
Schut et al.  & \textbf{1.000} $\pm$ \textbf{0.000} & 0.802 $\pm$ 0.201 & 5.203 $\pm$ 0.518 & 0.000 $\pm$ 0.000 & 0.997 $\pm$ 0.030 & 0.0007 $\pm$ 0.0006 & 0.0107 $\pm$ 0.0490 & \textbf{1.000} $\pm$ \textbf{0.002} & 0.0701 $\pm$ 0.0250 \\
\textbf{AVCG-B}        & \textbf{1.000} $\pm$ \textbf{0.000} & \textbf{0.407} $\pm$ \textbf{0.163} & \textbf{3.616} $\pm$ \textbf{0.123} & 0.670 $\pm$ 0.127 & \textbf{1.000} $\pm$ \textbf{0.000} & \textbf{0.0002} $\pm$ \textbf{0.0002} & 0.0218 $\pm$ 0.0195 & \textbf{1.000} $\pm$ \textbf{0.000} & \textbf{0.0011} $\pm$ \textbf{0.0001} \\
\textbf{AVCG-R ($\epsilon=0.8$)} & \textbf{1.000} $\pm$ \textbf{0.000} & 0.438 $\pm$ 0.194 & 3.646 $\pm$ 0.176 & \textbf{0.706} $\pm$ \textbf{0.140} & 0.996 $\pm$ 0.039 & \textbf{0.0002} $\pm$ \textbf{0.0003} & 0.0250 $\pm$ 0.0241 & 0.999 $\pm$ 0.011 & \textbf{0.0011} $\pm$ \textbf{0.0001} \\
\midrule
\multicolumn{10}{c}{\textit{Dataset: Spambase}} \\
\midrule
\multicolumn{10}{l}{\textbf{Environment: $\epsilon = 0.0$ (Strict Empirical Bound)}} \\
\midrule
Batten et al. & 0.926 $\pm$ 0.262 & 2.443 $\pm$ 2.198 & 9.521 $\pm$ 3.533 & 0.000 $\pm$ 0.000 & 0.755 $\pm$ 0.320 & 0.0159 $\pm$ 0.0203 & 0.0306 $\pm$ 0.1054 & 0.806 $\pm$ 0.268 & 0.0096 $\pm$ 0.0120 \\
QUCE          & 0.546 $\pm$ 0.498 & 2.084 $\pm$ 1.874 & 8.357 $\pm$ 2.860 & 0.000 $\pm$ 0.000 & 0.530 $\pm$ 0.416 & 0.0107 $\pm$ 0.0176 & \textbf{0.0024} $\pm$ \textbf{0.0096} & 0.541 $\pm$ 0.440 & 0.0118 $\pm$ 0.0138 \\
Schut et al.  & 0.990 $\pm$ 0.100 & 2.213 $\pm$ 1.742 & 9.878 $\pm$ 3.477 & 0.000 $\pm$ 0.000 & 0.773 $\pm$ 0.310 & 0.0127 $\pm$ 0.0279 & 0.0073 $\pm$ 0.0220 & 0.966 $\pm$ 0.120 & 0.0963 $\pm$ 0.1233 \\
\textbf{AVCG-B}        & \textbf{1.000} $\pm$ \textbf{0.000} & 0.983 $\pm$ 0.244 & \textbf{6.276} $\pm$ \textbf{0.091} & 0.776 $\pm$ 0.128 & \textbf{1.000} $\pm$ \textbf{0.000} & \textbf{0.0000} $\pm$ \textbf{0.0000} & 0.0271 $\pm$ 0.0295 & \textbf{1.000} $\pm$ \textbf{0.000} & 0.0011 $\pm$ 0.0001 \\
\textbf{AVCG-R ($\epsilon=0.0$)} & \textbf{1.000} $\pm$ \textbf{0.000} & \textbf{0.976} $\pm$ \textbf{0.199} & 6.301 $\pm$ 0.093 & \textbf{0.747} $\pm$ \textbf{0.149} & \textbf{1.000} $\pm$ \textbf{0.000} & \textbf{0.0000} $\pm$ \textbf{0.0000} & 0.0247 $\pm$ 0.0263 & \textbf{1.000} $\pm$ \textbf{0.000} & \textbf{0.0010} $\pm$ \textbf{0.0001} \\
\midrule
\multicolumn{10}{l}{\textbf{Environment: $\epsilon = 0.8$ (Relaxed Empirical Bound)}} \\
\midrule
Batten et al. & 0.916 $\pm$ 0.277 & 2.443 $\pm$ 2.198 & 9.521 $\pm$ 3.533 & 0.000 $\pm$ 0.000 & 0.755 $\pm$ 0.320 & 0.0161 $\pm$ 0.0240 & 0.0321 $\pm$ 0.0949 & 0.816 $\pm$ 0.183 & 0.0096 $\pm$ 0.0120 \\
QUCE          & 0.540 $\pm$ 0.498 & 2.084 $\pm$ 1.874 & 8.357 $\pm$ 2.860 & 0.000 $\pm$ 0.000 & 0.530 $\pm$ 0.416 & 0.0118 $\pm$ 0.0225 & \textbf{0.0028} $\pm$ \textbf{0.0130} & 0.544 $\pm$ 0.408 & 0.0118 $\pm$ 0.0138 \\
Schut et al.  & 0.990 $\pm$ 0.100 & 2.213 $\pm$ 1.742 & 9.878 $\pm$ 3.477 & 0.000 $\pm$ 0.000 & 0.773 $\pm$ 0.310 & 0.0118 $\pm$ 0.0252 & 0.0091 $\pm$ 0.0431 & 0.958 $\pm$ 0.103 & 0.0963 $\pm$ 0.1233 \\
\textbf{AVCG-B}        & \textbf{1.000} $\pm$ \textbf{0.000} & \textbf{0.983} $\pm$ \textbf{0.244} & \textbf{6.276} $\pm$ \textbf{0.091} & \textbf{0.776} $\pm$ \textbf{0.128} & \textbf{1.000} $\pm$ \textbf{0.000} & \textbf{0.0000} $\pm$ \textbf{0.0000} & 0.0266 $\pm$ 0.0298 & \textbf{1.000} $\pm$ \textbf{0.000} & \textbf{0.0011} $\pm$ \textbf{0.0001} \\
\textbf{AVCG-R ($\epsilon=0.8$)} & \textbf{1.000} $\pm$ \textbf{0.000} & 1.037 $\pm$ 0.238 & 6.332 $\pm$ 0.087 & 0.690 $\pm$ 0.108 & \textbf{1.000} $\pm$ \textbf{0.000} & \textbf{0.0000} $\pm$ \textbf{0.0000} & 0.0208 $\pm$ 0.0215 & \textbf{1.000} $\pm$ \textbf{0.000} & \textbf{0.0011} $\pm$ \textbf{0.0001} \\
\bottomrule
\end{tabular}
}
\end{table*}

Intuitively, minimal predictive variance within the defined model space leads to highly robust counterfactual instances.

\begin{example}[Empirical Example]
Taking the Rashomon set at $\epsilon = 0.10$ on the Adult Income dataset yields a measured predictive variance of $\sim0.002$, assume then AVCG-R produces a highly confident counterfactual such that $\mathbb{E}_{\theta\sim P_R}[f_\theta(\mathbf{x}^\prime)] = 0.95$. Given a new predictive model can be interpreted as an approximate draw from $P_{R}(\theta)$, we seek the probability that the counterfactual remains valid ($ f_{\theta^*}(\mathbf{x}^\prime) > 0.5$). Setting $\beta = 0.45$, we have: 
\begin{align*}
                &\mathbb{P}_{\theta^* \sim P_R} \Big( f_{\theta^*}(\mathbf{x}^\prime) > 0.95 - 0.45 \Big) \geq 1-\frac{0.002}{0.45^2} \\
                &\implies \mathbb{P}_{\theta^* \sim P_R} \Big( f_{\theta^*}(\mathbf{x}^\prime) > 0.5 \Big) \geq 0.99
\end{align*}
Thus, with a $99\%$ probability, a counterfactual evaluated by a newly sampled model $\theta^*$ drawn from the Rashomon set will remain valid.
\end{example}

\subsection{Computing AVCG}
We proceed by defining the implementation of our proposed AVCG. 

We derive gradients for the predictive term noting that gradients of the KL and proximity terms follow standard formulations and are computed via automatic differentiation. Thus, by the chain rule, we are optimizing our counterfactual generative model $g_{\phi}$ under the predictor $P_{\theta}$ over parameters $\theta$ via the computational graph defined by:
\begin{align*}
    \mathbb{E}_{\theta \sim P_\Phi(\theta), \mathbf{z} \sim q_{\psi}}\bigg[ \frac{\partial \ log\ P_{\theta}}{\partial \mathbf{x}^\prime} \cdot \frac{\partial g_{\phi}}{\partial \phi} \bigg],
\end{align*}
where $P_{\theta}$ is defined by a pre-trained Bayesian predictor.

To enable gradient propagation through stochastic latent samples, we adopt the reparameterisation trick \cite{DBLP:journals/corr/KingmaW13}, such that: 
\begin{align*}
    \mathbf{z} = \mu \ + \sigma \odot \kappa, \ \kappa \sim {N}(0, I). 
\end{align*}

Furthermore, the encoder is updated via the path: 
\begin{align*}
    \mathbb{E}_{\theta \sim P_\Phi(\theta)}\bigg[\mathbb{E}_{\mathbf{z} \sim q_{\psi}}\bigg[  \frac{\partial \ log\ P_{\theta}}{\partial \mathbf{x}^\prime} \cdot \frac{\partial g_{\phi}}{\partial \mathbf{z}} \cdot \frac{\partial \mathbf{z}}{\partial \psi} \bigg]\bigg].
\end{align*}
where using the reparameterization trick, gradients can propagate through $\mathbf{z}$ with respect to the encoders parameters $\psi$. Implementation details are provided in the \emph{supplementary material}.

\section{Experimental Setup}
We use a validation set for determining which models belong in the Rashomon set, and then utilize the test set for evaluating the quality of generated counterfactuals across performance metrics. For each dataset, we first split our data into 80\% training and 20\% test. We then take 15\% of our training data as a hold-out validation set. All experiments are run across 5 seeds across 100 instances per dataset, with the mean and standard deviation recorded. 

Experiments were conducted on a workstation equipped with an Intel Core i9-14900 CPU and an NVIDIA RTX 2000 Ada Generation GPU.

\subsection{Datasets}\label{sec:datasets}
For the experiments in this work, we focus on four benchmark datasets. We adopt the tabular datasets, namely the adult income dataset \cite{adult_2}, heart disease \cite{heart_disease_45}, Wisconsin breast cancer \cite{breast_cancer_wisconsin_diagnostic_17} and spambase dataset \cite{spambase_94} used frequently in counterfactual literature \cite{batten2025uncertaintyaware, Schut2021GeneratingIC, 10.5555/3709347.3743791}.

\subsection{Baselines}
We compare our proposed AVCG approach against two Bayesian-based counterfactual explainers: Schut et al. \cite{Schut2021GeneratingIC} and Batten et al. \cite{batten2025uncertaintyaware}. Finally, we compare against QUCE \cite{Duell_2024}, but we instead operate over the expected prediction as opposed to a single point estimate.   

\subsection{Metrics}\label{sec:metrics}

In this section, we formally define the evaluation metrics utilized in our experiments to assess the efficacy, plausibility, and stability of the generated counterfactuals. Let $\mathbf{x} \in \mathbb{R}^d$ be the factual instance, and let $\mathbf{x}^\prime \in \mathbb{R}^d$ be the generated counterfactual instance targeting class $y^\prime$. Let $y_{orig} = \arg\max_c \mathbb{E}_{\theta \sim P(\theta \mid D)}[f_\theta(\mathbf{x})_c]$ represent the original predicted class, where $c \in \{1, \ldots, C\}$ denotes a class. We denote the indicator function as $\mathbb{I}[\cdot]$, and a small constant for numerical stability as $\eta$. In practice the evaluation metrics are averaged over instances.

In this work we evaluate Validity (Val) \cite{Mothilal_2020}, IM1 \cite{10.1007/978-3-030-86520-7_40, Schut2021GeneratingIC}, Implausibility (Imp) \cite{energycf}, Diversity \cite{Mothilal_2020}, Robustness to Noisy Executions (NE) \cite{10.24963/ijcai.2024/894}, Robustness to Input Changes (R$_{IC}$) \cite{batten2025uncertaintyaware} and inference time. Furthermore, we propose two metrics, Cross Model Validity (CMV) evaluating the counterfactual validity across newly trained neural networks, and Rashomon Validity Ratio (RVR) the counterfactuals valid under a set of models defined by $\epsilon$. We provide the details for known metrics in the supplementary, and introduce our metrics as follows:



\textbf{Rashomon Validity Ratio (RVR)} quantifies the counterfactual's robustness specifically to model multiplicity. It evaluates the fraction of valid counterfactuals across the empirically derived Rashomon set (as per validation set evaluation), $\Theta_R(\epsilon)$:
$$\text{RVR}(\mathbf{x}^\prime, y^\prime, \epsilon) = \frac{1}{\vert\Theta_R(\epsilon)\vert} \sum_{\theta_k \in \Theta_R(\epsilon)} \mathbb{I} \left[ \arg\max_{c} f_{\theta_k}(\mathbf{x}^\prime)_c = y^\prime \right].$$

\textbf{Cross Model Validity (CMV)} evaluates the transferability of the counterfactual by testing it against a multiple surrogate models, $f_S$, trained on the identical data distribution:
$$\text{CMV}(\mathbf{x}^\prime, y^\prime) = \frac{1}{N} \sum^N_{s=1}\mathbb{I} \left[ \arg\max_{c} f_S(\mathbf{x}^\prime)_c = y^\prime \right].$$
for $N$ surrogate models. In this work we evaluate the average validity across $N=5$ separately trained models.

\section{Results}

\begin{figure*}[h]
    \centering
    \includegraphics[width=\textwidth, height=0.65\textheight, keepaspectratio]{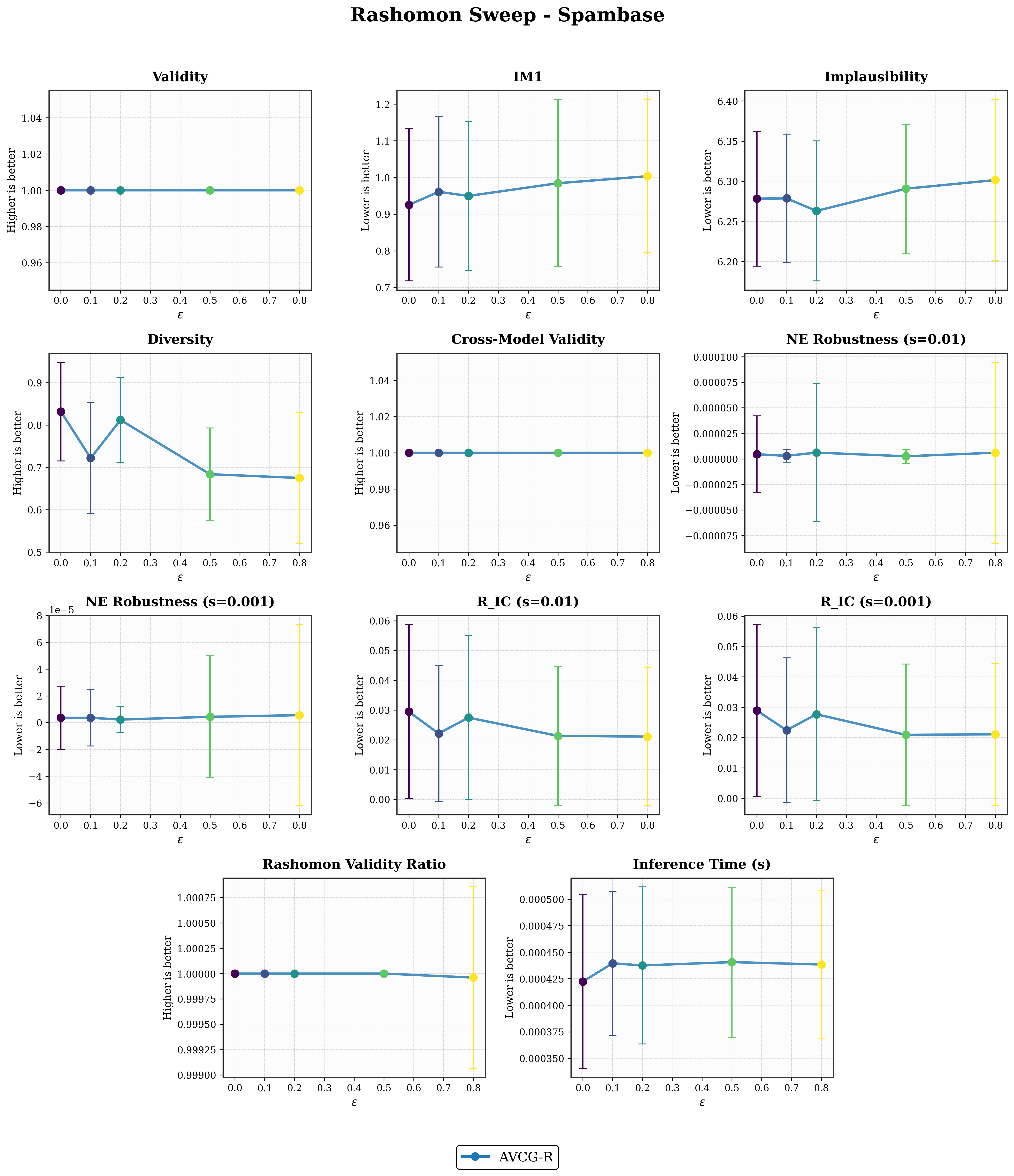}
    \caption{$\epsilon$-sweep on the Spambase dataset for AVCG-R. Here we observe the impact of the $\epsilon$ hyperparameter for the Rashomon set.}
    \label{fig:epsilon_sweep12}
\end{figure*}

\subsection{Results on Counterfactual Metrics}
\subsubsection{Robustness Under Predictive Uncertainty}
The primary objective of the AVCG framework is to generate counterfactuals that remain valid across a distribution of plausible predictive hypotheses. As evidenced by the Cross-Model Validity (CMV) and Rashomon Validity Ratio (RVR) in Tables \ref{tab:results_adult_breast1} and \ref{tab:results_heart_spam1}, standard post-hoc iterative methods are highly susceptible to model multiplicity. Several post-hoc methods exhibit substantial reductions in CMV on particular datasets. For instance, on the Spambase dataset (Table \ref{tab:results_heart_spam1}), Schut et al. achieves a CMV of only 0.773 and Batten et al. drops to 0.755. Furthermore, we observe that the RVR for non-AVCG methods drops. In contrast, both AVCG-B and AVCG-R maintain a cross-model validity ($\sim1.000$) and RVR ($\sim1.000$) across all four datasets. This is empirically consistent with Theorem~\ref{thrm:stability} and Corollary \ref{central_corol}.

\subsubsection{Amortization: Diversity and Inference Efficiency}
Beyond robustness, formulating counterfactual generation as a variational amortized process yields two practical advantages: inference speed and explanation diversity. First, the post-hoc baselines suffer from collapsing to  single solution, uniformly recording a Diversity score of $0.000$ across all datasets. Because they iteratively solve a deterministic optimization problem per instance, they collapse to a single, localized solution. AVCG, operating as a conditional generative model, maps inputs to a stochastic latent space $q_\psi(\mathbf{z} \mid y^\prime, \mathbf{x})$. By sampling different latent vectors $\mathbf{z}$, AVCG generates a diverse set of valid counterfactuals (e.g., Diversity of $1.078$ on Breast Cancer and $0.747$ on Spambase for AVCG-R), offering users multiple actionable paths for recourse. Secondly, the computational bottleneck of counterfactual deployment is inference latency. Iterative gradient-based baselines require significant time per instance (e.g., Schut et al. requires $0.04$ seconds per instance on Adult Income). By shifting the heavy computational burden to the training phase, AVCG requires only a single forward pass at inference time, achieving generation speeds on the order of $10^{-3}$ seconds which is an efficiency improvement of over an order of magnitude, making it highly suitable for real-time deployment. 

\subsection{$\epsilon$ Sweep}
In Figure~\ref{fig:epsilon_sweep12} we demonstrate the influence of the $\epsilon$ hyperparameter recording the mean and deviation across 5 runs on each metric under different Rashomon set sizes. The result demonstrates that AVCG-R remains stable across varying Rashomon set sizes, indicating that the framework is robust to the practitioner's choice of the empirical performance threshold $\epsilon$. 

We observe common trends, such as diversity commonly increasing across datasets as $\epsilon$ decreases. This is likely due to the higher epsilon which admits more models that need to evaluate a counterfactual as valid, and thus, limiting the space of potential counterfactual solutions. A similar consequence is likely for the IM1 metric as the constraints of validity become stricter.

We refer the reader to the \emph{supplementary material} for \emph{ablation studies}, additional baselines and further experiments on \emph{image datasets} (MNIST and CIFAR10).

\section{Conclusion}
We introduced AVCG, a generalized variational framework for amortized counterfactual generation over arbitrary hypothesis distributions. Bayesian posteriors and Rashomon-restricted hypothesis spaces arise as special cases. AVCG demonstrates superior performance across key metrics with the added value of amortized inference for a decreased computational cost.

An interesting direction for future work is to characterize the amortization gap induced by Rashomon-restricted inference, whereby the amortization gap is a known concern in VAE-like models \cite{Cremer2018InferenceSI}. Since the proposed generator amortizes counterfactual inference across both inputs and hypothesis-space, mismatch between the instance-optimal latent representations and the shared encoder may emerge as the Rashomon set evolves. Building on analyses of amortization error in variational inference, future work could investigate adaptive refinement strategies \cite{pmlr-v80-kim18e} that selectively update the encoder when hypothesis-space variability induces large deviations from instance-optimal counterfactual latent representations.

Future work can explore extrapolations of our generalized framework to further hypothesis spaces. For example, one could use a Laplace approximation around the maximum a posteriori (MAP) solution. 

\section*{Acknowledgments}
For the purpose of open access, the author has applied a Creative Commons Attribution (CC BY) licence to any Author Accepted Manuscript version of this paper, arising from this submission.

\bibliography{aaai2027}

@InProceedings{10.1007/978-3-030-86520-7_40,
author="Van Looveren, Arnaud
and Klaise, Janis",
editor="Oliver, Nuria
and P{\'e}rez-Cruz, Fernando
and Kramer, Stefan
and Read, Jesse
and Lozano, Jose A.",
title="Interpretable Counterfactual Explanations Guided by Prototypes",
booktitle="Machine Learning and Knowledge Discovery in Databases. Research Track",
year="2021",
publisher="Springer International Publishing",
address="Cham",
pages="650--665",
isbn="978-3-030-86520-7"
}

@misc{breast_cancer_wisconsin_diagnostic_17,
  author       = {Wolberg, William and Mangasarian, Olvi and Street, Nick and Street, W.},
  title        = {{Breast Cancer Wisconsin (Diagnostic)}},
  year         = {1993},
  howpublished = {UCI Machine Learning Repository},
  note         = {{DOI}: https://doi.org/10.24432/C5DW2B}
}

@misc{heart_disease_45,
  author       = {Janosi, Andras and Steinbrunn, William and Pfisterer, Matthias and Detrano, Robert},
  title        = {{Heart Disease}},
  year         = {1989},
  howpublished = {UCI Machine Learning Repository},
  note         = {{DOI}: https://doi.org/10.24432/C52P4X}
}

@String{Computing = "Computing" }

@String{Computer = "{IEEE} Computer" }

@String{Springer = "Springer-Verlag" }

@inproceedings{Schut2021GeneratingIC,
  title={Generating Interpretable Counterfactual Explanations By Implicit Minimisation of Epistemic and Aleatoric Uncertainties},
  author={Lisa Schut and Oscar Key and Rory McGrath and Luca Costabello and Bogdan Sacaleanu and Medb Corcoran and Yarin Gal},
  booktitle={International Conference on Artificial Intelligence and Statistics},
  year={2021},
  url={https://api.semanticscholar.org/CorpusID:232240104}
}

@inproceedings{10.5555/3709347.3743791,
author = {Noorani, Erfaun and Dissanayake, Pasan and Hamman, Faisal and Dutta, Sanghamitra},
title = {Counterfactual Explanations for Model Ensembles Using Entropic Risk Measures},
year = {2025},
isbn = {9798400714269},
publisher = {International Foundation for Autonomous Agents and Multiagent Systems},
address = {Richland, SC},
booktitle = {Proceedings of the 24th International Conference on Autonomous Agents and Multiagent Systems},
pages = {1566–1575},
numpages = {10},
location = {Detroit, MI, USA},
series = {AAMAS '25}
}

@inproceedings{DBLP:journals/corr/KingmaW13,
  author       = {Diederik P. Kingma and
                  Max Welling},
  editor       = {Yoshua Bengio and
                  Yann LeCun},
  title        = {Auto-Encoding Variational Bayes},
  booktitle    = {2nd International Conference on Learning Representations, {ICLR} 2014,
                  Banff, AB, Canada, April 14-16, 2014, Conference Track Proceedings},
  year         = {2014},
  url          = {http://arxiv.org/abs/1312.6114},
  bibsource    = {dblp computer science bibliography, https://dblp.org}
}

@article{wachter,
author = {Wachter, Sandra and Mittelstadt, Brent and Russell, Chris},
year = {2018},
month = {04},
pages = {841-887},
title = {Counterfactual Explanations Without Opening the Black Box: Automated Decisions and the GDPR},
volume = {31},
journal = {Harvard journal of law \& technology},
doi = {10.2139/ssrn.3063289}
}

@misc{sokol2025needcounterfactualexplainabilityprincipled,
      title={All You Need for Counterfactual Explainability Is Principled and Reliable Estimate of Aleatoric and Epistemic Uncertainty}, 
      author={Kacper Sokol and Eyke Hüllermeier},
      year={2025},
      eprint={2502.17007},
      archivePrefix={arXiv},
      primaryClass={cs.LG},
      url={https://arxiv.org/abs/2502.17007}, 
}

@misc{
batten2025uncertaintyaware,
title={Uncertainty-Aware Counterfactual Explanations using Bayesian Neural Nets},
author={Ben Batten and Francesco Leofante and Nicola Paoletti and Alessio Lomuscio and Mehran Hosseini},
year={2025},
url={https://openreview.net/forum?id=XAO5pulJru}
}

@inproceedings{Mothilal_2020, series={FAT* ’20},
   title={Explaining machine learning classifiers through diverse counterfactual explanations},
   url={http://dx.doi.org/10.1145/3351095.3372850},
   DOI={10.1145/3351095.3372850},
   booktitle={Proceedings of the 2020 Conference on Fairness, Accountability, and Transparency},
   publisher={ACM},
   author={Mothilal, Ramaravind K. and Sharma, Amit and Tan, Chenhao},
   year={2020},
   month=jan, pages={607–617},
   collection={FAT* ’20} }

@InProceedings{pmlr-v80-kim18e,
  title = 	 {Semi-Amortized Variational Autoencoders},
  author =       {Kim, Yoon and Wiseman, Sam and Miller, Andrew and Sontag, David and Rush, Alexander},
  booktitle = 	 {Proceedings of the 35th International Conference on Machine Learning},
  pages = 	 {2678--2687},
  year = 	 {2018},
  editor = 	 {Dy, Jennifer and Krause, Andreas},
  volume = 	 {80},
  series = 	 {Proceedings of Machine Learning Research},
  month = 	 {10--15 Jul},
  publisher =    {PMLR},
  url = 	 {https://proceedings.mlr.press/v80/kim18e.html}
}

@inproceedings{Cremer2018InferenceSI,
  title={Inference Suboptimality in Variational Autoencoders},
  author={Chris Cremer and Xuechen Li and David Kristjanson Duvenaud},
  booktitle={International Conference on Machine Learning},
  year={2018},
  url={https://api.semanticscholar.org/CorpusID:3524184}
}

@misc{spambase_94,
  author       = {Hopkins, Mark and Reeber, Erik and Forman, George and Suermondt, Jaap},
  title        = {{Spambase}},
  year         = {1999},
  howpublished = {UCI Machine Learning Repository},
  note         = {{DOI}: https://doi.org/10.24432/C53G6X}
}

@misc{adult_2,
  author       = {Becker, Barry and Kohavi, Ronny},
  title        = {{Adult}},
  year         = {1996},
  howpublished = {UCI Machine Learning Repository},
  note         = {{DOI}: https://doi.org/10.24432/C5XW20}
}

@inproceedings{DBLP:conf/iclr/HsuLH024,
  author       = {Hsiang Hsu and
                  Guihong Li and
                  Shaohan Hu and
                  Chun{-}Fu Chen},
  title        = {Dropout-Based Rashomon Set Exploration for Efficient Predictive Multiplicity
                  Estimation},
  booktitle    = {The Twelfth International Conference on Learning Representations,
                  {ICLR} 2024, Vienna, Austria, May 7-11, 2024},
  publisher    = {OpenReview.net},
  year         = {2024},
  url          = {https://openreview.net/forum?id=Sf2A2PUXO3},
  bibsource    = {dblp computer science bibliography, https://dblp.org}
}

@inproceedings{10.5555/3045390.3045502,
author = {Gal, Yarin and Ghahramani, Zoubin},
title = {Dropout as a Bayesian approximation: representing model uncertainty in deep learning},
year = {2016},
publisher = {JMLR.org},
booktitle = {Proceedings of the 33rd International Conference on International Conference on Machine Learning - Volume 48},
pages = {1050–1059},
numpages = {10},
location = {New York, NY, USA},
series = {ICML'16}
}

@inproceedings{10.5555/3045118.3045290,
author = {Blundell, Charles and Cornebise, Julien and Kavukcuoglu, Koray and Wierstra, Daan},
title = {Weight uncertainty in neural networks},
year = {2015},
publisher = {JMLR.org},
booktitle = {Proceedings of the 32nd International Conference on Machine Learning - Volume 37},
pages = {1613–1622},
numpages = {10},
location = {Lille, France},
series = {ICML'15}
}

@inproceedings{NIPS2015_8d55a249,
 author = {Sohn, Kihyuk and Lee, Honglak and Yan, Xinchen},
 booktitle = {Advances in Neural Information Processing Systems},
 editor = {C. Cortes and N. Lawrence and D. Lee and M. Sugiyama and R. Garnett},
 pages = {},
 publisher = {Curran Associates, Inc.},
 title = {Learning Structured Output Representation using Deep Conditional Generative Models},
 url = {https://proceedings.neurips.cc/paper_files/paper/2015/file/8d55a249e6baa5c06772297520da2051-Paper.pdf},
 volume = {28},
 year = {2015}
}

@inproceedings{10.24963/ijcai.2024/894,
author = {Jiang, Junqi and Leofante, Francesco and Rago, Antonio and Toni, Francesca},
title = {Robust counterfactual explanations in machine learning: a survey},
year = {2024},
isbn = {978-1-956792-04-1},
url = {https://doi.org/10.24963/ijcai.2024/894},
doi = {10.24963/ijcai.2024/894},
booktitle = {Proceedings of the Thirty-Third International Joint Conference on Artificial Intelligence},
articleno = {894},
numpages = {9},
location = {Jeju, Korea},
series = {IJCAI '24}
}

@ArtifactSoftware{R,
    title = {R: A Language and Environment for Statistical Computing},
    author = {{R Core Team}},
    organization = {R Foundation for Statistical Computing},
    address = {Vienna, Austria},
    year = {2019},
    url = {https://www.R-project.org/},
}

@inproceedings{
hsu2026the,
title={The Rashomon Set Has It All: Analyzing Trustworthiness of Trees under Multiplicity},
author={Ethan Hsu and Tony Cao and Lesia Semenova and Chudi Zhong},
booktitle={The Thirty-ninth Annual Conference on Neural Information Processing Systems Datasets and Benchmarks Track},
year={2025},
url={https://openreview.net/forum?id=RXsDPn3toF}
}

@article{energycf,
author = {Altmeyer, Patrick and Farmanbar, Mojtaba and Deursen, Arie and Liem, Cynthia},
year = {2024},
month = {03},
pages = {10829-10837},
title = {Faithful Model Explanations through Energy-Constrained Conformal Counterfactuals},
volume = {38},
journal = {Proceedings of the AAAI Conference on Artificial Intelligence},
doi = {10.1609/aaai.v38i10.28956}
}

@inproceedings{Duell_2024,
   title={QUCE: The Minimisation and Quantification of Path-Based Uncertainty for Generative Counterfactual Explanations},
   url={http://dx.doi.org/10.1109/ICDM59182.2024.00078},
   DOI={10.1109/icdm59182.2024.00078},
   booktitle={2024 IEEE International Conference on Data Mining (ICDM)},
   publisher={IEEE},
   author={Duell, Jamie and Seisenberger, Monika and Fu, Hsuan and Fan, Xiuyi},
   year={2024},
   month=dec, pages={693–698} }

\newpage
\onecolumn

\appendix

\section{Appendix}
\subsection{Proof of Theorem 1}
\begin{proof}
Since $R_{\Phi}(\mathbf{x}^\prime)$ is the only term dependent on $\theta$, we reduce the difference to the terms
 \begin{align*}
    &\mathbb{E}_{q_{\psi}(\mathbf{z} \mid y^\prime, \mathbf{x})}
    \bigg[
    \mathbb{E}_{\theta \sim P^a_{\Phi}(\theta)}
    \bigg[
    \log P_{\theta}(y^\prime \mid g_{\phi}(\mathbf{x}, y^\prime, \mathbf{z}))
    \bigg]
    \bigg] \\
    &-
    \mathbb{E}_{q_{\psi}(\mathbf{z} \mid y^\prime, \mathbf{x})}
    \bigg[
    \mathbb{E}_{\theta \sim P^b_{\Phi}(\theta)}
    \bigg[
    \log P_{\theta}(y^\prime \mid g_{\phi}(\mathbf{x}, y^\prime, \mathbf{z}))
    \bigg]
    \bigg] \\
    &= L_{a}(\mathbf{x}^\prime)-L_{b}(\mathbf{x}^\prime)\\
    &=
    \int
    q_{\psi}(\mathbf{z}\mid y^\prime,\mathbf{x})
    \left(
    \int
    P^a_{\Phi}(\theta)
    \log P_{\theta}(y^\prime \mid g_{\phi}(\mathbf{x},y^\prime,\mathbf{z}))
    \,d\theta
    \right)
    d\mathbf{z}
    \\
    &\quad-
    \int
    q_{\psi}(\mathbf{z}\mid y^\prime,\mathbf{x})
    \left(
    \int
    P^b_{\Phi}(\theta)
    \log P_{\theta}(y^\prime \mid g_{\phi}(\mathbf{x},y^\prime,\mathbf{z}))
    \,d\theta
    \right)
    d\mathbf{z}
    \\
    &=
    \int
    q_{\psi}(\mathbf{z}\mid y^\prime,\mathbf{x})
 \left(
    \int
    \left(
    P^a_{\Phi}(\theta)-P^b_{\Phi}(\theta)
    \right)
    \log P_{\theta}(y^\prime \mid g_{\phi}(\mathbf{x},y^\prime,\mathbf{z}))
    \,d\theta
    \right)
    d\mathbf{z}.
\end{align*}
Taking the absolute values we have: 
\begin{align*}
    \vert L_{a}(\mathbf{x}^\prime)-L_{b}(\mathbf{x}^\prime) \vert &= \bigg\vert \int
    q_{\psi}(\mathbf{z}\mid y^\prime,\mathbf{x})
    \left(
    \int
    \left(
    P^a_{\Phi}(\theta)-P^b_{\Phi}(\theta)
    \right)
    \log P_{\theta}(y^\prime \mid g_{\phi}(\mathbf{x},y^\prime,\mathbf{z}))
    \,d\theta
    \right)
    d\mathbf{z}\bigg\vert 
    \\ &\leq \int
    q_{\psi}(\mathbf{z}\mid y^\prime,\mathbf{x})
     \left(
    \int
    \bigg\vert \left(
    P^a_{\Phi}(\theta)-P^b_{\Phi}(\theta)
    \right)
    \log P_{\theta}(y^\prime \mid g_{\phi}(\mathbf{x},y^\prime,\mathbf{z}))
    \,\bigg\vert  d\theta
    \right)
    d\mathbf{z}
    \\
    &\leq \int
    q_{\psi}(\mathbf{z}\mid y^\prime,\mathbf{x})
     \left(
   \int
    \bigg\vert \left(
    P^a_{\Phi}(\theta)-P^b_{\Phi}(\theta)
    \right)\vert 
    \vert \log P_{\theta}(y^\prime \mid g_{\phi}(\mathbf{x},y^\prime,\mathbf{z}))
    \,\bigg\vert  d\theta
    \right)
    d\mathbf{z}
\end{align*}
let $\vert \log P_{\theta}(y^\prime \mid g_{\phi}(\mathbf{x},y^\prime,\mathbf{z}))\vert \leq M$ then we have: 
\begin{align*}
    \int
    q_{\psi}(\mathbf{z}\mid y^\prime,\mathbf{x})
    \left(M
    \int
    \bigg\vert \left(
    P^a_{\Phi}(\theta)-P^b_{\Phi}(\theta)
    \right)
    \bigg\vert  d\theta
    \right)
    d\mathbf{z}
\end{align*}

    Using the definition of total variation $TV(P,Q) = \frac{1}{2}\int \vert P - Q\vert $ for two distribution $P$ and $Q$, then we obtain: 
\begin{align*}
    \vert L_{a}(\mathbf{x}^\prime)-L_{b}(\mathbf{x}^\prime) \vert &\leq 2 M
    TV(P^a_{\Phi}, P^b_{\Phi})
    \int
    q_{\psi}(\mathbf{z}\mid y^\prime,\mathbf{x})
    d\mathbf{z} \\
    &= 2 M
    TV(P^a_{\Phi}, P^b_{\Phi})
\end{align*}
since $\int
    q_{\psi}(\mathbf{z}\mid y^\prime,\mathbf{x})
    d\mathbf{z}= 1$, applying Pinskers inequality we have: 
    \begin{align*}
    &2 M
    TV(P^a_{\Phi}, P^b_{\Phi}) \leq 2M\sqrt{\frac{1}{2}D_{KL}(P^a_{\Phi} \vert\vert P^b_{\Phi})}  \\
    &\implies \vert L_{a}(\mathbf{x}^\prime)-L_{b}(\mathbf{x}^\prime) \vert \leq \sqrt{2}M\sqrt{D_{KL}(P^a_{\Phi} \vert\vert P^b_{\Phi})}
    \end{align*}
    completing the proof.
\end{proof}

\subsection{Proof of Proposition 1}
\begin{proof}
Let $X = 1- f_{\theta}(\mathbf{x}^\prime)$ and $\alpha = 1-\gamma$ be a random variable over $\theta$ s.t. $X \geq 0$. Applying Markov's inequality:
\begin{align*}
    \mathbb{P}(X \geq \alpha) \leq \frac{\mathbb{E}[X]}{\alpha}
\end{align*}
Substituting terms yields:
\begin{align*}
    \mathbb{P}(1- f_{\theta}(\mathbf{x}^\prime) \geq 1-\gamma) \leq\frac{1-\mathbb{E}_{\theta \sim P_{\Phi}(\theta)}[f_{\theta}(\mathbf{x}^\prime)]}{1 -\gamma} 
\end{align*}
Taking the complement, we observe: 
\begin{align*}
    \mathbb{P}(f_{\theta}(\mathbf{x}^\prime) \geq \gamma) \geq 1- \frac{1-\mathbb{E}_{\theta \sim P_{\Phi}(\theta)}[f_{\theta}(\mathbf{x}^\prime)]}{1 -\gamma}.
\end{align*}
\end{proof}

\subsection{Proof of Corollary 1}
\begin{proof}
    Chebyshev's inequality: 
    \begin{align*}
\mathbb{P}_{\theta \sim P_\Phi} \Big( \vert  f_{\theta}(\mathbf{x}^\prime) - \mathbb{E}_{\theta\sim P_\Phi}[f_\theta(\mathbf{x}^\prime)] \vert < \beta \Big) \ge 1-\frac{\sigma^2}{\beta^2}.
\end{align*}
    then, we observe that: 
    \begin{align*}
        &\vert f_{\theta}(\mathbf{x}^\prime) - \mathbb{E}_{\theta\sim P_R}[f_\theta(\mathbf{x}^\prime)] \vert < \beta \\ &\implies -\beta < f_{\theta}(\mathbf{x}^\prime) - \mathbb{E}_{\theta\sim P_R}[f_\theta(\mathbf{x}^\prime)] < \beta
        \\ &\implies -\beta + \mathbb{E}_{\theta\sim P_R}[f_\theta(\mathbf{x}^\prime)]< f_{\theta}(\mathbf{x}^\prime) < \beta + \mathbb{E}_{\theta\sim P_R}[f_\theta(\mathbf{x}^\prime)]
    \end{align*}
    By isolating the LHS of the inequality and substituting $\theta$ for $\theta^*$, under the assumption $\theta^* \sim P_{R}(\theta)$: 
    \begin{align*}
        f_{\theta^*}(\mathbf{x}^\prime) > \mathbb{E}_{\theta^* \sim P_{R}(\theta)}[f_\theta(\mathbf{x}^\prime)] - \beta, 
    \end{align*}
    and thus:
    \begin{align*}
            \mathbb{P}_{\theta^* \sim P_{R}(\theta)} \Big( f_{\theta^*}(\mathbf{x}^\prime) > \mathbb{E}_{\theta^* \sim P_{R}(\theta)}[f_\theta(\mathbf{x}^\prime)] - \beta \Big) \ge 1-\frac{\sigma^2}{\beta^2}.
    \end{align*}
\end{proof}

\subsection{Related Preliminaries}
Our AVCG approach is naturally inspired by the architectures of Variational Autoencoders (VAE) \cite{DBLP:journals/corr/KingmaW13} and Conditional Variational Autoencoders (CVAE) \cite{NIPS2015_8d55a249}. For completeness, we show the derived Evidence Lower Bound (ELBO) for each of the methods.   

\subsection{Variational Autoencoder}
The Variational Autoencoder (VAE) \cite{DBLP:journals/corr/KingmaW13} is a foundational approach for amortized inference in generative models.  The premise of the VAE is to maximise the log-evidence of an instance $\mathbf{x}$ (log $p(\mathbf{x})$), this is mathematically expressed and optimised for via the Evidence Lower Bound (ELBO) as follows: 
\begin{align*}
    \text{log }p(\mathbf{x}) &= \text{log }\int p(\mathbf{x} \mid \mathbf{z})p(\mathbf{z})d\mathbf{z} \\ &= \text{log } \int q(\mathbf{z} \mid \mathbf{x})\frac{p(\mathbf{x} \mid \mathbf{z})p(\mathbf{z})}{q(\mathbf{z} \mid \mathbf{x})}d\mathbf{z}
    \\ &\geq \int q(\mathbf{z} \mid \mathbf{x}) \text{log }\frac{p(\mathbf{x} \mid \mathbf{z})p(\mathbf{z})}{q(\mathbf{z} \mid \mathbf{x})}d\mathbf{z}
    \\ &= \mathbb{E}_{q(\mathbf{z} \mid \mathbf{x})}\bigg[ \text{log } p(\mathbf{x} \mid \mathbf{z}) \bigg] - D_{KL}\bigg( q(\mathbf{z} \vert \mathbf{x}) \mid\mid p(\mathbf{z})\bigg).  
\end{align*}

\subsection{Conditional Variational Autoencoder}
The Conditional Variational Autoencoder (CVAE) \cite{NIPS2015_8d55a249} imposes a conditional constraint $\mathbf{c}$, such that: 
\begin{align*}
    \text{log }p(\mathbf{x} \mid \mathbf{c}) &= \text{log }\int p(\mathbf{x} \mid \mathbf{z}, \mathbf{c})p(\mathbf{z} \mid \mathbf{c})d\mathbf{z} \\ &= \text{log } \int q(\mathbf{z} \mid \mathbf{x},\mathbf{c})\frac{p(\mathbf{x} \mid \mathbf{z}, \mathbf{c})p(\mathbf{z} \mid \mathbf{c})}{q(\mathbf{z} \mid \mathbf{x}, \mathbf{c})}d\mathbf{z}
    \\ &\geq \int q(\mathbf{z} \mid \mathbf{x}, \mathbf{c}) \text{log }\frac{p(\mathbf{x} \mid \mathbf{z}, \mathbf{c})p(\mathbf{z} \mid \mathbf{c})}{q(\mathbf{z} \mid \mathbf{x}, \mathbf{c})}d\mathbf{z}
    \\ &= \mathbb{E}_{q(\mathbf{z} \mid \mathbf{x}, \mathbf{c})}\bigg[ \text{log } p(\mathbf{x} \mid \mathbf{z}, \mathbf{c}) \bigg] \\ &- D_{KL}\bigg( q(\mathbf{z} \vert \mathbf{x}, \mathbf{c}) \mid\mid p(\mathbf{z} \mid \mathbf{c})\bigg).  
\end{align*}

\subsection{Metrics}
\textbf{Validity} measures whether the generated counterfactual successfully alters the prediction of the primary Bayesian predictor, such that:
$$\text{Validity}(\mathbf{x}^\prime, y^\prime) = \mathbb{I} \left[ \arg\max_{c} \mathbb{E}_{\theta \sim P(\theta \mid D)}[f_\theta(\mathbf{x}^\prime)_c] = y^\prime \right].$$
\textbf{Implausibility (Imp)} calculates the average Euclidean distance from the generated counterfactual to all training instances belonging to the target class, denoted by the set $D_{y^\prime}$:

$$\text{Imp}(\mathbf{x}^\prime) = \frac{1}{\vert D_{y^\prime}\vert} \sum_{\mathbf{x}_i \in D_{y^\prime}} \vert\vert \mathbf{x}^\prime - \mathbf{x}_i \vert\vert_2.$$

\textbf{IM1 Score} evaluates manifold alignment, let $AE_c(\cdot)$ denote the reconstruction by an autoencoder trained exclusively on instances of class $c$: 

$$\text{IM1}(\mathbf{x}^\prime, y_{orig}, y^\prime) = \frac{\vert\vert \mathbf{x}^\prime - AE_{y^\prime}(\mathbf{x}^\prime) \vert\vert_2^2}{\vert\vert \mathbf{x}^\prime - AE_{y_{orig}}(\mathbf{x}^\prime) \vert\vert_2^2 + \eta}.$$

\textbf{Robustness to Noisy Executions (NE)} measures the smoothness of the model's posterior predictive distribution around the counterfactual. We inject Gaussian noise $\delta$ into $\mathbf{x}^\prime$ and measure the predictive discrepancy:
$$\text{NE}(\mathbf{x}^\prime) = \mathbb{E}_{\delta} \left[ \Big\Vert \mathbb{E}_{\theta}[f_\theta(\mathbf{x}^\prime)] - \mathbb{E}_{\theta}[f_\theta(\mathbf{x}^\prime + \delta)] \Big\Vert_2^2 \right].$$

\textbf{Robustness to Input Changes ($R_{IC}$)} evaluates the stability of the counterfactual generator itself, $g_{\phi}(\cdot)$. We subject the original input $\mathbf{x}$ to an input perturbation $\delta$:

$$\text{R}_{IC}(\mathbf{x}, \sigma) = \frac{\vert\vert g_{\phi}(\mathbf{x} + \delta, y^\prime) - g_{\phi}(\mathbf{x}, y^\prime) \vert\vert_2^2}{\vert\vert g_{\phi}(\mathbf{x}, y^\prime) - \mathbf{x} \vert\vert_2^2 + \eta}.$$

\subsection{Implementation Details}
 In practice, we use automatic differentiation in PyTorch. Expectations over the Rashomon-restricted posterior are approximated using 50 Monte Carlo dropout forward passes per instance providing an empirical hypothesis-space approximation since MC Dropout is used defining a Rashomon set providing predictive multiplicity estimation \cite{DBLP:conf/iclr/HsuLH024}. 

Latent expectations are estimated using 5 samples from  $q_{\psi}(\mathbf{z} \mid \mathbf{x}, y^\prime)$. We approximate the Rashomon-restricted posterior using an empirical uniform distribution over retained MC-dropout predictors satisfying the validation loss constraint under a deterministic set of $\epsilon$, following recent results showing that dropout masks produce models that lie within the Rashomon set with high probability. \cite{DBLP:conf/iclr/HsuLH024}

\subsection{Image Dataset Results}
In Table \ref{tab:results_images} we present additional results on MNIST and CIFAR10 to show our method generalises to image data. 
\begin{table*}[t]
\centering
\caption{Amortized Counterfactual Evaluation on High-Dimensional Image Datasets (Mean $\pm$ std over 5 seeds). }
\label{tab:results_images}
\resizebox{\textwidth}{!}{
\begin{tabular}{l ccc ccccc c}
\toprule
\textbf{Method} & \textbf{Val} $\uparrow$ & \textbf{IM1} $\downarrow$ & \textbf{Imp} $\downarrow$ & \textbf{Div} $\uparrow$ & \textbf{CMV} $\uparrow$ & \textbf{NE} $\downarrow$ & \textbf{R$_{IC}$} $\downarrow$ & \textbf{RVR} $\uparrow$ & \textbf{Time (s)} $\downarrow$ \\
\midrule
\multicolumn{10}{c}{\textit{Dataset: CIFAR-10}} \\
\midrule
\multicolumn{10}{l}{\textbf{Environment: $\epsilon = 0.0$ (Strict Empirical Bound)}} \\
\midrule
Batten et al. & 0.940 $\pm$ 0.238 & 1.115 $\pm$ 0.189 & 93.594 $\pm$ 13.378 & 0.000 $\pm$ 0.000 & 0.068 $\pm$ 0.175 & 0.0097 $\pm$ 0.0089 & 0.4985 $\pm$ 0.1547 & 0.774 $\pm$ 0.418 & 0.0022 $\pm$ 0.0009 \\
QUCE          & 0.928 $\pm$ 0.259 & 1.082 $\pm$ 0.449 & 90.446 $\pm$ 12.771 & 0.000 $\pm$ 0.000 & 0.254 $\pm$ 0.307 & 0.0082 $\pm$ 0.0058 & 0.0481 $\pm$ 0.0391 & 0.622 $\pm$ 0.485 & 0.0114 $\pm$ 0.0059 \\
Schut et al.  & 0.944 $\pm$ 0.230 & 1.103 $\pm$ 0.157 & 94.061 $\pm$ 13.460 & 0.000 $\pm$ 0.000 & 0.116 $\pm$ 0.233 & 0.0126 $\pm$ 0.0131 & 0.4461 $\pm$ 0.4465 & 0.900 $\pm$ 0.300 & 1.8211 $\pm$ 0.5513 \\
\textbf{AVCG-B}        & \textbf{1.000} $\pm$ \textbf{0.000} & \textbf{1.036} $\pm$ \textbf{0.166} & 67.878 $\pm$ 5.515 & 6.033 $\pm$ 1.056 & \textbf{0.774} $\pm$ \textbf{0.305} & \textbf{0.0000} $\pm$ \textbf{0.0000} & \textbf{0.0103} $\pm$ \textbf{0.0092} & \textbf{1.000} $\pm$ \textbf{0.000} & \textbf{0.0013} $\pm$ \textbf{0.0001} \\
\textbf{AVCG-R ($\epsilon=0.0$)} & 0.998 $\pm$ 0.045 & 1.085 $\pm$ 0.252 & \textbf{67.673} $\pm$ \textbf{5.397} & \textbf{6.485} $\pm$ \textbf{2.069} & 0.324 $\pm$ 0.380 & 0.0130 $\pm$ 0.0157 & 0.0129 $\pm$ 0.0165 & \textbf{1.000} $\pm$ \textbf{0.000} & 0.0014 $\pm$ 0.0001 \\
\midrule
\multicolumn{10}{l}{\textbf{Environment: $\epsilon = 0.8$ (Relaxed Empirical Bound)}} \\
\midrule
Batten et al. & 0.926 $\pm$ 0.262 & 1.115 $\pm$ 0.189 & 93.582 $\pm$ 13.467 & 0.000 $\pm$ 0.000 & 0.068 $\pm$ 0.175 & 0.0103 $\pm$ 0.0098 & 0.5035 $\pm$ 0.1615 & 0.743 $\pm$ 0.227 & 0.0022 $\pm$ 0.0009 \\
QUCE          & 0.940 $\pm$ 0.238 & 1.082 $\pm$ 0.449 & 90.492 $\pm$ 12.757 & 0.000 $\pm$ 0.000 & 0.254 $\pm$ 0.307 & 0.0085 $\pm$ 0.0056 & 0.0460 $\pm$ 0.0408 & 0.631 $\pm$ 0.186 & 0.0114 $\pm$ 0.0059 \\
Schut et al.  & 0.952 $\pm$ 0.214 & 1.103 $\pm$ 0.157 & 94.145 $\pm$ 13.347 & 0.000 $\pm$ 0.000 & 0.116 $\pm$ 0.233 & 0.0122 $\pm$ 0.0140 & 0.4385 $\pm$ 0.4415 & 0.880 $\pm$ 0.204 & 1.8211 $\pm$ 0.5513 \\
\textbf{AVCG-B}        & \textbf{1.000} $\pm$ \textbf{0.000} & \textbf{1.036} $\pm$ \textbf{0.166} & 67.854 $\pm$ 5.622 & 6.033 $\pm$ 1.056 & \textbf{0.774} $\pm$ \textbf{0.305} & \textbf{0.0000} $\pm$ \textbf{0.0000} & \textbf{0.0105} $\pm$ \textbf{0.0105} & \textbf{1.000} $\pm$ \textbf{0.000} & \textbf{0.0013} $\pm$ \textbf{0.0001} \\
\textbf{AVCG-R ($\epsilon=0.8$)} & \textbf{1.000} $\pm$ \textbf{0.000} & 1.050 $\pm$ 0.175 & \textbf{67.602} $\pm$ \textbf{5.572} & \textbf{6.422 $\pm$ 1.194} & 0.513 $\pm$ 0.389 & 0.0018 $\pm$ 0.0047 & 0.0115 $\pm$ 0.0106 & \textbf{1.000} $\pm$ \textbf{0.000} & 0.0014 $\pm$ 0.0001 \\
\midrule
\multicolumn{10}{c}{\textit{Dataset: MNIST}} \\
\midrule
\multicolumn{10}{l}{\textbf{Environment: $\epsilon = 0.0$ (Strict Empirical Bound)}} \\
\midrule
Batten et al. & 0.834 $\pm$ 0.372 & 1.410 $\pm$ 0.408 & 34.425 $\pm$ 2.324 & 0.000 $\pm$ 0.000 & 0.023 $\pm$ 0.102 & 0.0159 $\pm$ 0.0145 & 0.6965 $\pm$ 0.1516 & 0.670 $\pm$ 0.470 & 0.0074 $\pm$ 0.0021 \\
QUCE          & 0.890 $\pm$ 0.313 & 1.392 $\pm$ 0.622 & 30.619 $\pm$ 2.075 & 0.000 $\pm$ 0.000 & 0.326 $\pm$ 0.343 & 0.0131 $\pm$ 0.0117 & 0.0372 $\pm$ 0.0405 & 0.620 $\pm$ 0.485 & 0.0310 $\pm$ 0.0148 \\
Schut et al.  & 0.982 $\pm$ 0.133 & 1.337 $\pm$ 0.419 & 34.737 $\pm$ 2.480 & 0.000 $\pm$ 0.000 & 0.252 $\pm$ 0.298 & 0.0178 $\pm$ 0.0213 & 0.5169 $\pm$ 0.4543 & 0.928 $\pm$ 0.259 & 1.3342 $\pm$ 0.4253 \\
\textbf{AVCG-B}        & \textbf{1.000} $\pm$ \textbf{0.000} & \textbf{1.218} $\pm$ \textbf{0.577} & \textbf{26.871} $\pm$ \textbf{1.557} & \textbf{2.489} $\pm$ \textbf{0.241} & \textbf{0.891} $\pm$ \textbf{0.245} & \textbf{0.0000} $\pm$ \textbf{0.0000} & 0.0138 $\pm$ 0.0053 & \textbf{1.000} $\pm$ \textbf{0.000} & \textbf{0.0013} $\pm$ \textbf{0.0001} \\
\textbf{AVCG-R ($\epsilon=0.0$)} & \textbf{1.000} $\pm$ \textbf{0.000} & 1.367 $\pm$ 0.778 & 26.975 $\pm$ 1.577 & 2.442 $\pm$ 0.247 & 0.482 $\pm$ 0.375 & 0.0031 $\pm$ 0.0045 & \textbf{0.0138} $\pm$ \textbf{0.0059} & \textbf{1.000} $\pm$ \textbf{0.000} & 0.0014 $\pm$ 0.0001 \\
\midrule
\multicolumn{10}{l}{\textbf{Environment: $\epsilon = 0.8$ (Relaxed Empirical Bound)}} \\
\midrule
Batten et al. & 0.800 $\pm$ 0.400 & 1.410 $\pm$ 0.408 & 34.430 $\pm$ 2.309 & 0.000 $\pm$ 0.000 & 0.023 $\pm$ 0.102 & 0.0157 $\pm$ 0.0131 & 0.7023 $\pm$ 0.1521 & 0.588 $\pm$ 0.198 & 0.0074 $\pm$ 0.0021 \\
QUCE          & 0.882 $\pm$ 0.323 & 1.392 $\pm$ 0.622 & 30.602 $\pm$ 2.115 & 0.000 $\pm$ 0.000 & 0.326 $\pm$ 0.343 & 0.0120 $\pm$ 0.0105 & 0.0377 $\pm$ 0.0570 & 0.605 $\pm$ 0.207 & 0.0310 $\pm$ 0.0148 \\
Schut et al.  & 0.984 $\pm$ 0.126 & 1.337 $\pm$ 0.419 & 34.734 $\pm$ 2.477 & 0.000 $\pm$ 0.000 & 0.252 $\pm$ 0.298 & 0.0173 $\pm$ 0.0173 & 0.4975 $\pm$ 0.4247 & 0.909 $\pm$ 0.133 & 1.3342 $\pm$ 0.4253 \\
\textbf{AVCG-B}        & \textbf{1.000} $\pm$ \textbf{0.000} & \textbf{1.218} $\pm$ \textbf{0.577} & \textbf{26.870} $\pm$ \textbf{1.501} & \textbf{2.489} $\pm$ \textbf{0.241} & \textbf{0.891} $\pm$ \textbf{0.245} & \textbf{0.0000} $\pm$ \textbf{0.0000} & \textbf{0.0138} $\pm$ \textbf{0.0055} & \textbf{1.000} $\pm$ \textbf{0.000} & \textbf{0.0013} $\pm$ \textbf{0.0001} \\
\textbf{AVCG-R ($\epsilon=0.8$)} & \textbf{1.000} $\pm$ \textbf{0.000} & 1.243 $\pm$ 0.631 & 26.944 $\pm$ 1.596 & 2.369 $\pm$ 0.262 & 0.823 $\pm$ 0.273 & \textbf{0.0000} $\pm$ \textbf{0.0000} & 0.0124 $\pm$ 0.0053 & \textbf{1.000} $\pm$ \textbf{0.000} & 0.0014 $\pm$ 0.0001 \\
\bottomrule
\end{tabular}
}
\end{table*}

\section{Additional $\epsilon$-Sweep Results}
In Figures~\ref{fig:epsilon_sweep1}--\ref{fig:epsilon_sweep2} we provide additional $\epsilon$-sweeps across all Tabular datasets presented in the main paper.
\begin{figure*}[h]
    \centering
    \includegraphics[width=\textwidth, height=0.55\textheight, keepaspectratio]{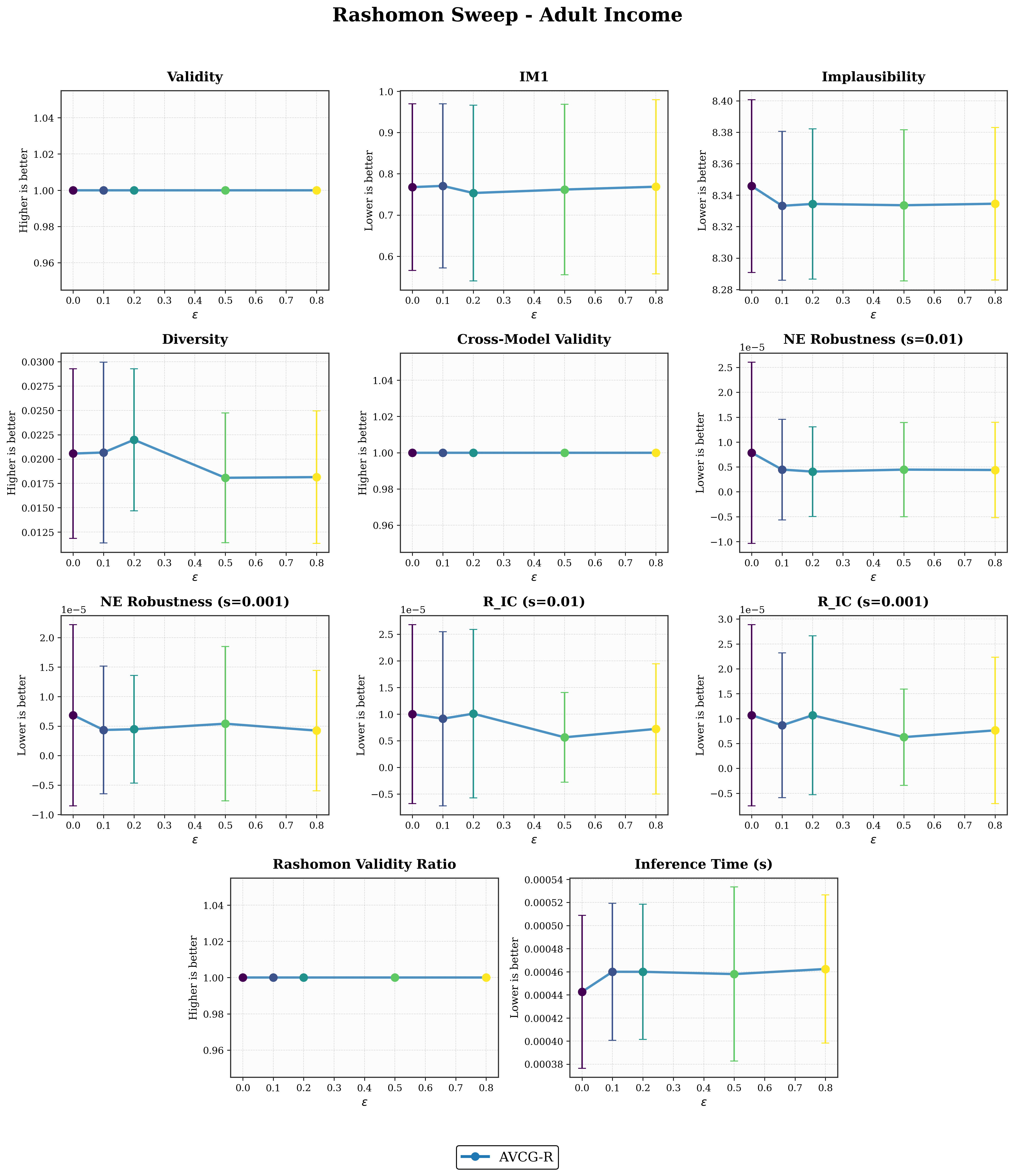}
    \caption{$\epsilon$-sweep on the adult income dataset for AVCG-R. Here we observe the impact of the $\epsilon$ hyperparameter for the Rashomon set.}
    \label{fig:epsilon_sweep1}
\end{figure*}
\begin{figure*}[h]
    \centering
    \includegraphics[width=\textwidth, height=0.55\textheight, keepaspectratio]{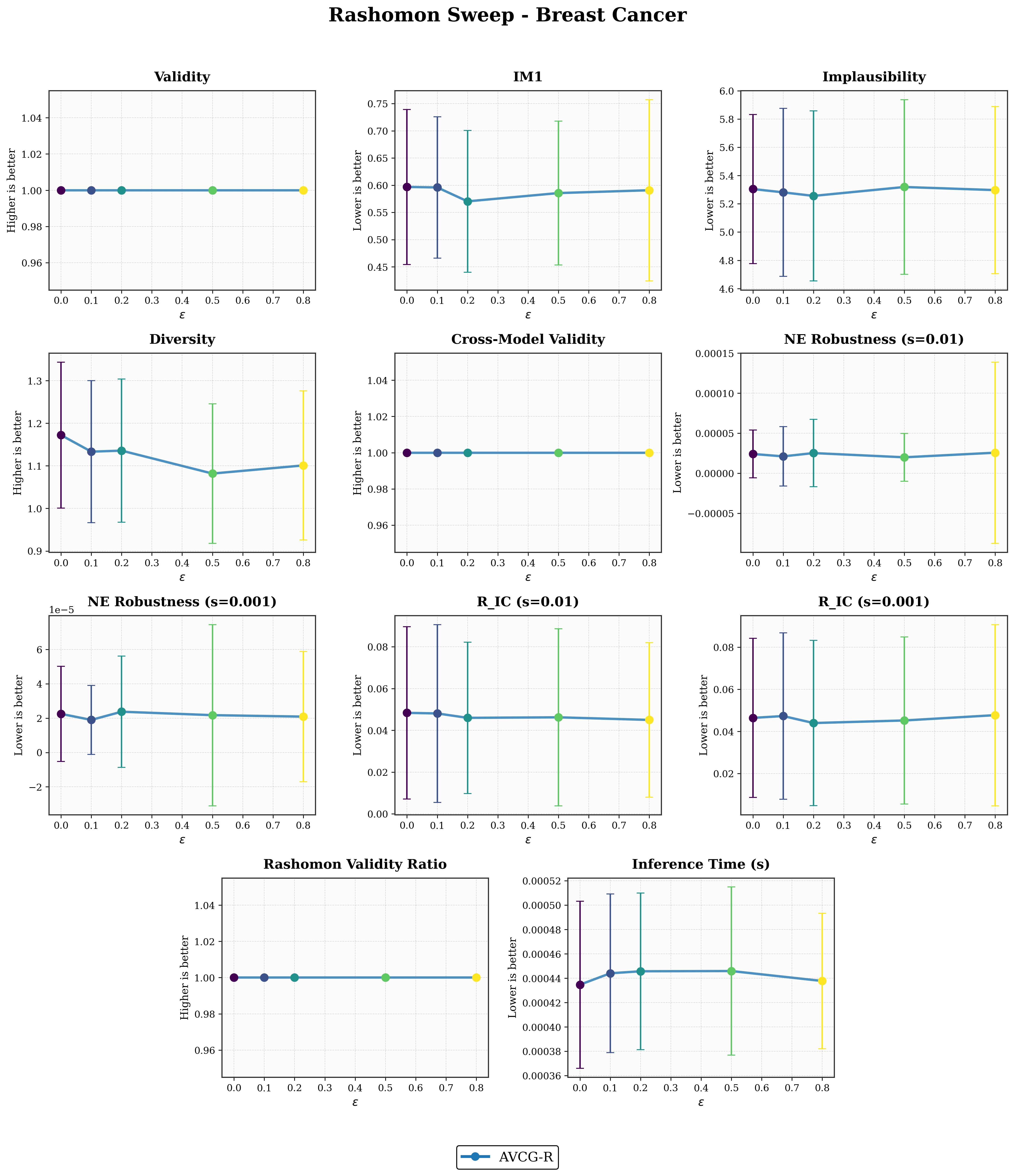}
    \caption{$\epsilon$-sweep on the breast cancer dataset for AVCG-R. Here we observe the impact of the $\epsilon$ hyperparameter for the Rashomon set.}
    \label{fig:epsilon_sweep}
\end{figure*}

\begin{figure*}[h]
    \centering
    \includegraphics[width=\textwidth, height=0.55\textheight, keepaspectratio]{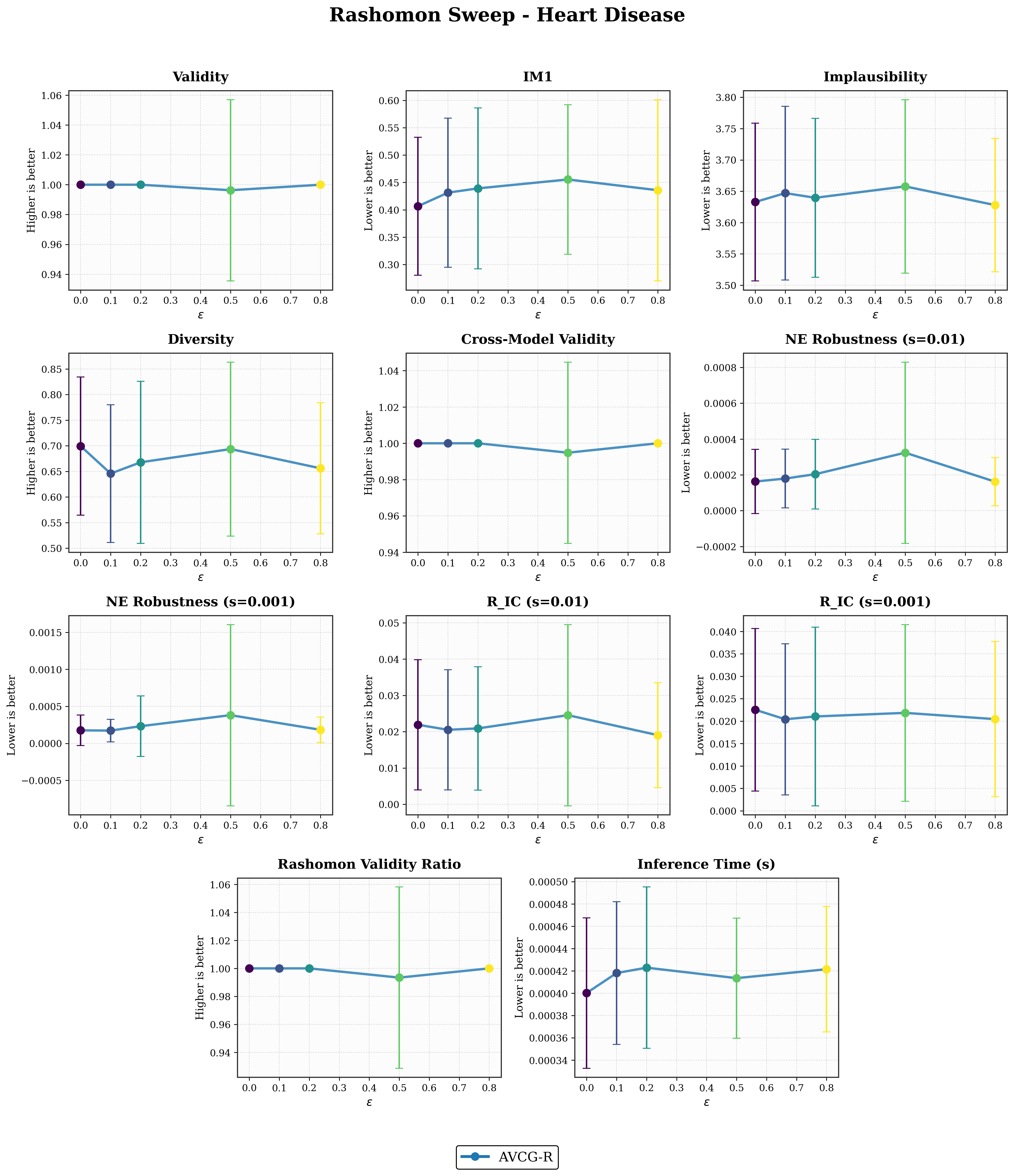}
    \caption{$\epsilon$-sweep on the heart disease dataset for AVCG-R. Here we observe the impact of the $\epsilon$ hyperparameter for the Rashomon set.}
    \label{fig:epsilon_sweep2}
\end{figure*}

\section{Ablations}
In Figure~\ref{fig:ablation_1} and Figure~\ref{fig:ablation_2} we provide ablation studies on each term across each metric.
\begin{figure*}
    \centering
    \includegraphics[width=0.5\linewidth]{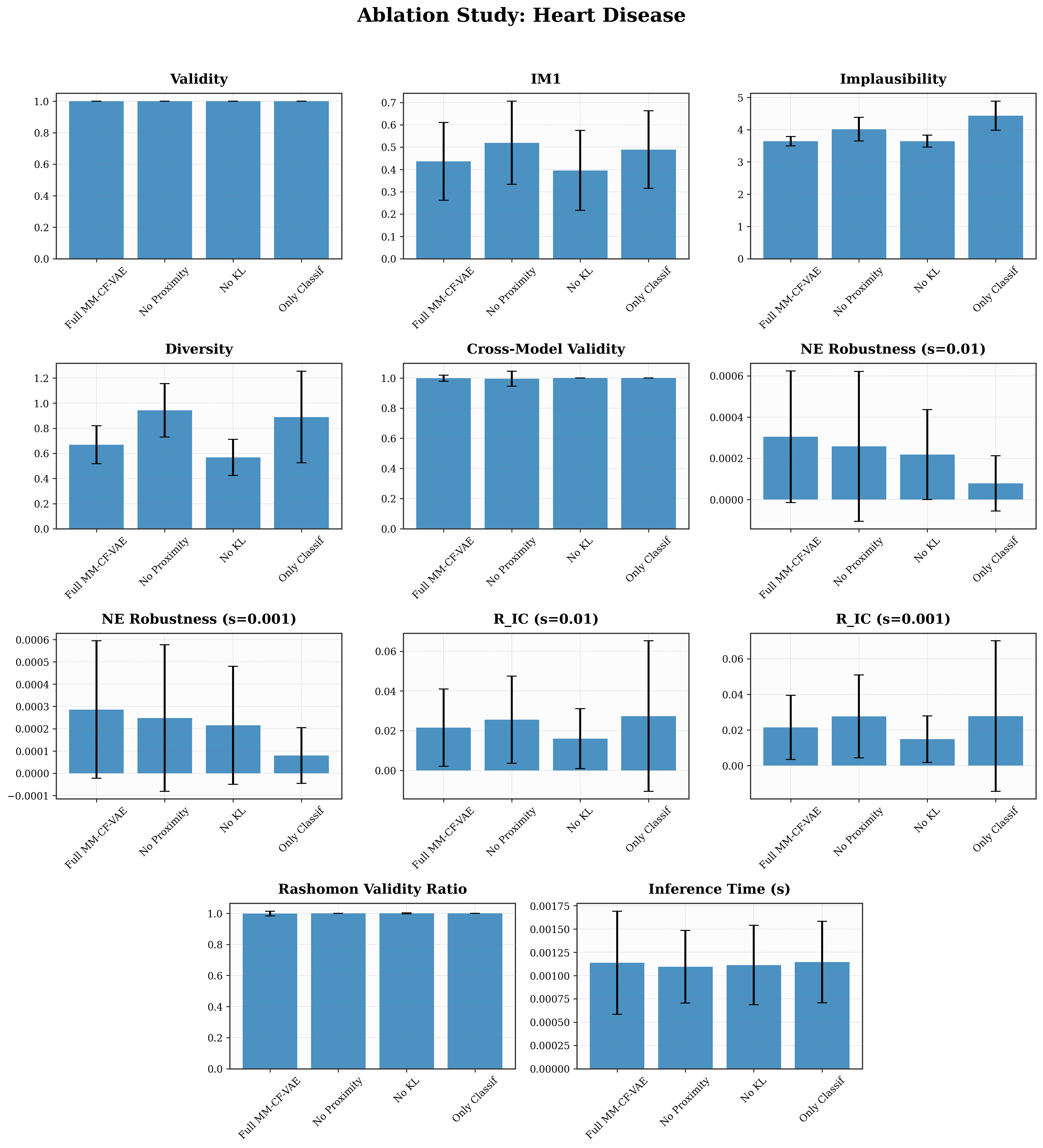}
    \caption{Ablation study on the heart disease dataset across key metrics.}
    \label{fig:ablation_1}
\end{figure*}

\begin{figure*}
    \centering
    \includegraphics[width=0.5\linewidth]{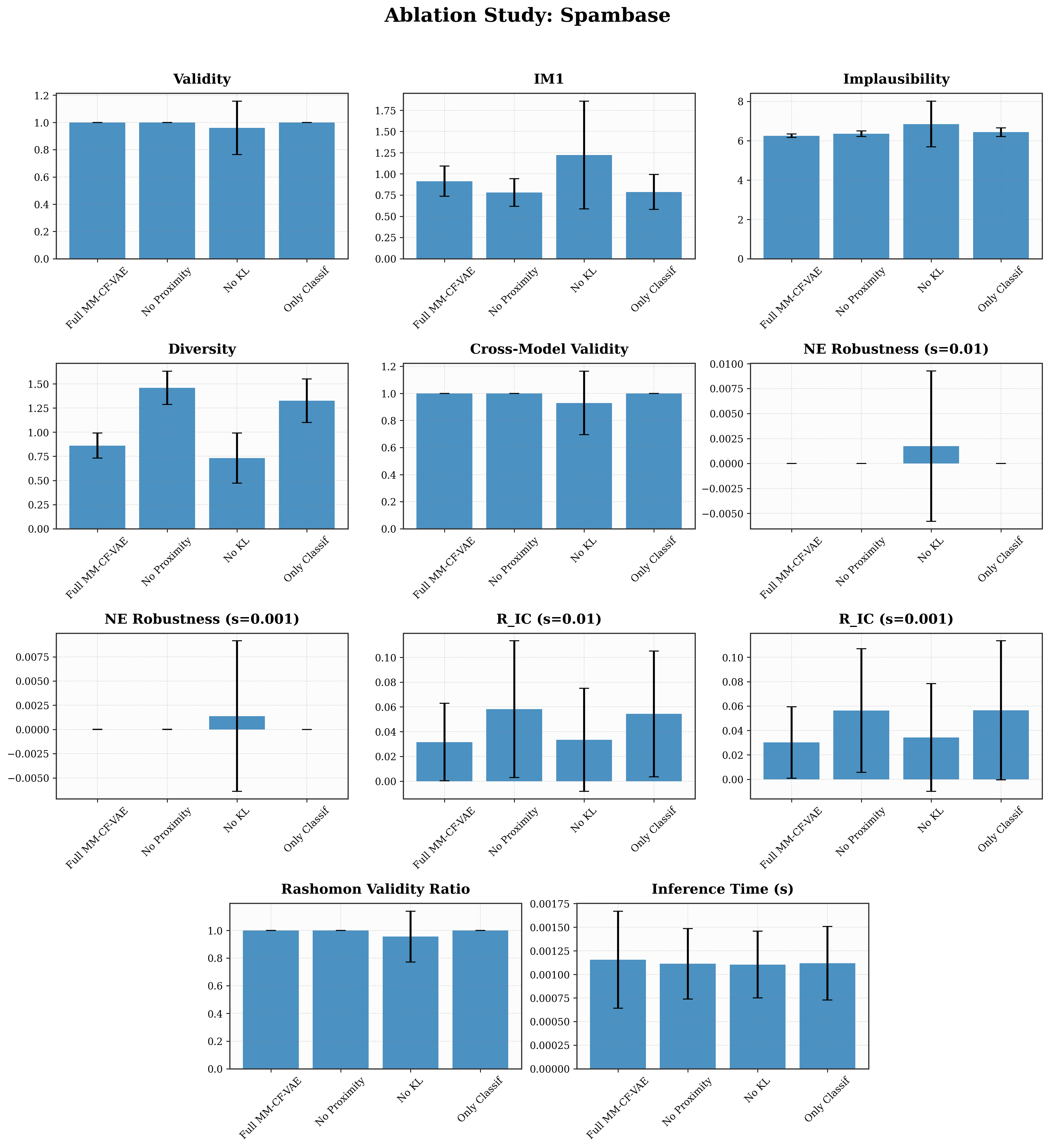}
    \caption{Ablation study on the spambase dataset across key metrics.}
    \label{fig:ablation_2}
\end{figure*}

\section{Additional Methods}
In this section we provide a brief comparison of the standard Conditional VAE (CVAE) inferencing $p(\mathbf{x} \mid y^\prime, \mathbf{z})$, where $p(\mathbf{x} \mid y, \mathbf{z})$ is the learned standard CVAE. We also include Wachter et al. \cite{wachter}. 

\begin{table*}[t]
\centering
\caption{Amortized Counterfactual Evaluation on Adult Income and Breast Cancer (Mean $\pm$ std over 5 seeds).}
\label{tab:results_adult_breast}
\resizebox{\textwidth}{!}{
\begin{tabular}{l ccc cccccc c}
\toprule
\textbf{Method} & \textbf{Val} $\uparrow$ & \textbf{IM1} $\downarrow$ & \textbf{Imp} $\downarrow$ & \textbf{Div} $\uparrow$ & \textbf{CMV} $\uparrow$ & \textbf{NE} $\downarrow$ & \textbf{R$_{IC}$} $\downarrow$ & \textbf{RVR} $\uparrow$ & \textbf{Time (s)} $\downarrow$ \\
\midrule
\multicolumn{10}{c}{\textit{Dataset: Adult Income}} \\
\midrule
\multicolumn{10}{l}{\textbf{Environment: $\epsilon = 0.0$ (Strict Empirical Bound)}} \\
\midrule
CF-VAE        & 0.226 $\pm$ 0.418 & 1.404 $\pm$ 0.966 & 10.246 $\pm$ 2.855 & \textbf{2.332} $\pm$ \textbf{0.568} & 0.224 $\pm$ 0.378 & 0.0009 $\pm$ 0.0029 & 0.1553 $\pm$ 0.1150 & 0.206 $\pm$ 0.393 & 0.0004 $\pm$ 0.0001 \\
Wachter et al.& 0.762 $\pm$ 0.426 & 2.531 $\pm$ 3.570 & 11.712 $\pm$ 3.317 & 0.000 $\pm$ 0.000 & 0.550 $\pm$ 0.319 & 0.0074 $\pm$ 0.0106 & 0.0087 $\pm$ 0.0271 & 0.396 $\pm$ 0.368 & 0.0603 $\pm$ 0.1166 \\
\textbf{AVCG-B}        & \textbf{1.000} $\pm$ \textbf{0.000} & \textbf{0.756} $\pm$ \textbf{0.202} & \textbf{8.328} $\pm$ \textbf{0.047} & 0.027 $\pm$ 0.009 & \textbf{1.000} $\pm$ \textbf{0.000} & \textbf{0.0000} $\pm$ \textbf{0.0000} & \textbf{0.0000} $\pm$ \textbf{0.0000} & \textbf{1.000} $\pm$ \textbf{0.000} & \textbf{0.0003} $\pm$ \textbf{0.0000} \\
\textbf{AVCG-R ($\epsilon=0.0$)} & \textbf{1.000} $\pm$ \textbf{0.000} & 0.768 $\pm$ 0.202 & 8.346 $\pm$ 0.055 & 0.021 $\pm$ 0.009 & \textbf{1.000} $\pm$ \textbf{0.000} & \textbf{0.0000} $\pm$ \textbf{0.0000} & \textbf{0.0000} $\pm$ \textbf{0.0000} & \textbf{1.000} $\pm$ \textbf{0.000} & 0.0004 $\pm$ 0.0001 \\
\midrule
\multicolumn{10}{l}{\textbf{Environment: $\epsilon = 0.8$ (Relaxed Empirical Bound)}} \\
\midrule
CF-VAE        & 0.232 $\pm$ 0.422 & 1.404 $\pm$ 0.966 & 10.246 $\pm$ 2.855 & \textbf{2.332} $\pm$ \textbf{0.568} & 0.224 $\pm$ 0.378 & 0.0009 $\pm$ 0.0030 & 0.1570 $\pm$ 0.1215 & 0.222 $\pm$ 0.390 & 0.0004 $\pm$ 0.0001 \\
Wachter et al.& 0.792 $\pm$ 0.406 & 2.531 $\pm$ 3.570 & 11.712 $\pm$ 3.317 & 0.000 $\pm$ 0.000 & 0.550 $\pm$ 0.319 & 0.0065 $\pm$ 0.0074 & 0.0096 $\pm$ 0.0312 & 0.549 $\pm$ 0.122 & 0.0603 $\pm$ 0.1166 \\
\textbf{AVCG-B}        & \textbf{1.000} $\pm$ \textbf{0.000} & \textbf{0.756} $\pm$ \textbf{0.202} & \textbf{8.328} $\pm$ \textbf{0.047} & 0.027 $\pm$ 0.009 & \textbf{1.000} $\pm$ \textbf{0.000} & \textbf{0.0000} $\pm$ \textbf{0.0000} & \textbf{0.0000} $\pm$ \textbf{0.0000} & \textbf{1.000} $\pm$ \textbf{0.000} & \textbf{0.0003} $\pm$ \textbf{0.0000} \\
\textbf{AVCG-R ($\epsilon=0.8$)} & \textbf{1.000} $\pm$ \textbf{0.000} & 0.769 $\pm$ 0.211 & \textbf{8.335} $\pm$ \textbf{0.049} & 0.018 $\pm$ 0.007 & \textbf{1.000} $\pm$ \textbf{0.000} & \textbf{0.0000} $\pm$ \textbf{0.0000} & \textbf{0.0000} $\pm$ \textbf{0.0000} & \textbf{1.000} $\pm$ \textbf{0.000} & 0.0005 $\pm$ 0.0001 \\
\midrule
\multicolumn{10}{c}{\textit{Dataset: Breast Cancer}} \\
\midrule
\multicolumn{10}{l}{\textbf{Environment: $\epsilon = 0.0$ (Strict Empirical Bound)}} \\
\midrule
CF-VAE        & 0.414 $\pm$ 0.493 & 2.909 $\pm$ 2.912 & 6.704 $\pm$ 1.094 & \textbf{2.031} $\pm$ \textbf{0.471} & 0.380 $\pm$ 0.460 & 0.0016 $\pm$ 0.0036 & 0.2866 $\pm$ 0.2297 & 0.411 $\pm$ 0.471 & 0.0011 $\pm$ 0.0003 \\
Wachter et al.& 0.748 $\pm$ 0.434 & 1.609 $\pm$ 0.827 & 7.260 $\pm$ 1.236 & 0.000 $\pm$ 0.000 & 0.335 $\pm$ 0.319 & 0.0072 $\pm$ 0.0076 & 0.0008 $\pm$ 0.0024 & 0.583 $\pm$ 0.184 & 0.0640 $\pm$ 0.0565 \\
\textbf{AVCG-B}        & \textbf{1.000} $\pm$ \textbf{0.000} & \textbf{0.573} $\pm$ \textbf{0.136} & 5.341 $\pm$ 0.543 & 1.132 $\pm$ 0.162 & \textbf{1.000} $\pm$ \textbf{0.000} & \textbf{0.0000} $\pm$ \textbf{0.0003} & 0.0473 $\pm$ 0.0361 & \textbf{1.000} $\pm$ \textbf{0.000} & \textbf{0.0003} $\pm$ \textbf{0.0000} \\
\textbf{AVCG-R ($\epsilon=0.0$)} & \textbf{1.000} $\pm$ \textbf{0.000} & 0.597 $\pm$ 0.143 & \textbf{5.305} $\pm$ \textbf{0.527} & 1.172 $\pm$ 0.171 & \textbf{1.000} $\pm$ \textbf{0.000} & \textbf{0.0000} $\pm$ \textbf{0.0000} & 0.0484 $\pm$ 0.0413 & \textbf{1.000} $\pm$ \textbf{0.000} & 0.0004 $\pm$ 0.0001 \\
\midrule
\multicolumn{10}{l}{\textbf{Environment: $\epsilon = 0.8$ (Relaxed Empirical Bound)}} \\
\midrule
CF-VAE        & 0.414 $\pm$ 0.493 & 2.908 $\pm$ 2.912 & 6.704 $\pm$ 1.094 & \textbf{2.031} $\pm$ \textbf{0.471} & 0.380 $\pm$ 0.460 & 0.0016 $\pm$ 0.0036 & 0.3048 $\pm$ 0.2510 & 0.417 $\pm$ 0.460 & 0.0011 $\pm$ 0.0003 \\
Wachter et al.& 0.768 $\pm$ 0.422 & 1.609 $\pm$ 0.827 & 7.260 $\pm$ 1.236 & 0.000 $\pm$ 0.000 & 0.335 $\pm$ 0.319 & 0.0075 $\pm$ 0.0079 & \textbf{0.0008} $\pm$ \textbf{0.0020} & 0.576 $\pm$ 0.067 & 0.0640 $\pm$ 0.0565 \\
\textbf{AVCG-B}        & \textbf{1.000} $\pm$ \textbf{0.000} & \textbf{0.573} $\pm$ \textbf{0.136} & 5.341 $\pm$ 0.543 & \textbf{1.132} $\pm$ \textbf{0.162} & \textbf{1.000} $\pm$ \textbf{0.000} & \textbf{0.0000} $\pm$ \textbf{0.0000} & 0.0466 $\pm$ 0.0397 & \textbf{1.000} $\pm$ \textbf{0.000} & \textbf{0.0003} $\pm$ \textbf{0.0000} \\
\textbf{AVCG-R ($\epsilon=0.8$)} & \textbf{1.000} $\pm$ \textbf{0.000} & 0.591 $\pm$ 0.167 & \textbf{5.297} $\pm$ \textbf{0.591} & 1.101 $\pm$ 0.175 & \textbf{1.000} $\pm$ \textbf{0.000} & \textbf{0.0000} $\pm$ \textbf{0.0000} & 0.0450 $\pm$ 0.0370 & \textbf{1.000} $\pm$ \textbf{0.000} & 0.0004 $\pm$ 0.0001 \\
\bottomrule
\end{tabular}
}
\end{table*}

\begin{table*}[t]
\centering
\caption{Amortized Counterfactual Evaluation on Heart Disease and Spambase (Mean $\pm$ std over 5 seeds).}
\label{tab:results_heart_spam}
\resizebox{\textwidth}{!}{
\begin{tabular}{l ccc cccccc c}
\toprule
\textbf{Method} & \textbf{Val} $\uparrow$ & \textbf{IM1} $\downarrow$ & \textbf{Imp} $\downarrow$ & \textbf{Div} $\uparrow$ & \textbf{CMV} $\uparrow$ & \textbf{NE} $\downarrow$ & \textbf{R$_{IC}$} $\downarrow$ & \textbf{RVR} $\uparrow$ & \textbf{Time (s)} $\downarrow$ \\
\midrule
\multicolumn{10}{c}{\textit{Dataset: Heart Disease}} \\
\midrule
\multicolumn{10}{l}{\textbf{Environment: $\epsilon = 0.0$ (Strict Empirical Bound)}} \\
\midrule
CF-VAE        & 0.218 $\pm$ 0.412 & 3.682 $\pm$ 5.346 & 4.160 $\pm$ 0.462 & \textbf{1.282} $\pm$ \textbf{0.281} & 0.212 $\pm$ 0.390 & 0.0006 $\pm$ 0.0005 & 0.1212 $\pm$ 0.1098 & 0.223 $\pm$ 0.387 & 0.0011 $\pm$ 0.0001 \\
Wachter et al.& 0.581 $\pm$ 0.493 & 1.141 $\pm$ 0.358 & 4.808 $\pm$ 0.456 & 0.000 $\pm$ 0.000 & 0.408 $\pm$ 0.340 & 0.0034 $\pm$ 0.0033 & \textbf{0.0040} $\pm$ \textbf{0.0240} & 0.500 $\pm$ 0.221 & 0.0408 $\pm$ 0.0223 \\
\textbf{AVCG-B}        & \textbf{1.000} $\pm$ \textbf{0.000} & 0.407 $\pm$ \textbf{0.163} & \textbf{3.616} $\pm$ \textbf{0.123} & 0.670 $\pm$ 0.127 & \textbf{1.000} $\pm$ \textbf{0.000} & \textbf{0.0002} $\pm$ \textbf{0.0002} & 0.0245 $\pm$ 0.0201 & \textbf{1.000} $\pm$ \textbf{0.000} & \textbf{0.0003} $\pm$ \textbf{0.0001} \\
\textbf{AVCG-R ($\epsilon=0.0$)} & \textbf{1.000} $\pm$ \textbf{0.000} & \textbf{0.407} $\pm$ 0.126 & 3.633 $\pm$ 0.126 & 0.699 $\pm$ 0.135 & \textbf{1.000} $\pm$ \textbf{0.000} & \textbf{0.0002} $\pm$ \textbf{0.0002} & 0.0219 $\pm$ 0.0180 & \textbf{1.000} $\pm$ \textbf{0.000} & 0.0004 $\pm$ 0.0001 \\
\midrule
\multicolumn{10}{l}{\textbf{Environment: $\epsilon = 0.8$ (Relaxed Empirical Bound)}} \\
\midrule
CF-VAE        & 0.226 $\pm$ 0.418 & 3.682 $\pm$ 5.346 & 4.160 $\pm$ 0.461 & \textbf{1.282} $\pm$ \textbf{0.281} & 0.212 $\pm$ 0.390 & 0.0006 $\pm$ 0.0005 & 0.1255 $\pm$ 0.1109 & 0.223 $\pm$ 0.376 & 0.0011 $\pm$ 0.0001 \\
Wachter et al.& 0.581 $\pm$ 0.493 & 1.141 $\pm$ 0.358 & 4.808 $\pm$ 0.456 & 0.000 $\pm$ 0.000 & 0.408 $\pm$ 0.340 & 0.0033 $\pm$ 0.0032 & \textbf{0.0037} $\pm$ \textbf{0.0190} & 0.517 $\pm$ 0.067 & 0.0408 $\pm$ 0.0223 \\
\textbf{AVCG-B}        & \textbf{1.000} $\pm$ \textbf{0.000} & \textbf{0.407} $\pm$ 0.163 & \textbf{3.616} $\pm$ \textbf{0.123} & 0.670 $\pm$ 0.127 & \textbf{1.000} $\pm$ \textbf{0.000} & \textbf{0.0002} $\pm$ \textbf{0.0002} & 0.0261 $\pm$ 0.0341 & \textbf{1.000} $\pm$ \textbf{0.004} & \textbf{0.0003} $\pm$ \textbf{0.0001} \\
\textbf{AVCG-R ($\epsilon=0.8$)} & \textbf{1.000} $\pm$ \textbf{0.000} & 0.436 $\pm$ 0.166 & 3.628 $\pm$ 0.107 & 0.656 $\pm$ 0.128 & \textbf{1.000} $\pm$ \textbf{0.000} & \textbf{0.0002} $\pm$ \textbf{0.0001} & 0.0190 $\pm$ 0.0144 & \textbf{1.000} $\pm$ \textbf{0.000} & 0.0004 $\pm$ 0.0001 \\
\midrule
\multicolumn{10}{c}{\textit{Dataset: Spambase}} \\
\midrule
\multicolumn{10}{l}{\textbf{Environment: $\epsilon = 0.0$ (Strict Empirical Bound)}} \\
\midrule
CF-VAE        & 0.334 $\pm$ 0.472 & 3.159 $\pm$ 3.534 & 8.248 $\pm$ 2.909 & \textbf{2.520} $\pm$ \textbf{0.499} & 0.346 $\pm$ 0.437 & 0.0060 $\pm$ 0.0205 & 0.3544 $\pm$ 0.2724 & 0.334 $\pm$ 0.446 & 0.0011 $\pm$ 0.0001 \\
Wachter et al.& 0.732 $\pm$ 0.443 & 2.750 $\pm$ 2.440 & 9.535 $\pm$ 3.474 & 0.000 $\pm$ 0.000 & 0.495 $\pm$ 0.331 & 0.0227 $\pm$ 0.0188 & \textbf{0.0054} $\pm$ \textbf{0.0178} & 0.566 $\pm$ 0.252 & 0.0301 $\pm$ 0.0448 \\
\textbf{AVCG-B}        & \textbf{1.000} $\pm$ \textbf{0.000} & 0.983 $\pm$ 0.244 & \textbf{6.276} $\pm$ \textbf{0.091} & 0.744 $\pm$ 0.107 & \textbf{1.000} $\pm$ \textbf{0.000} & \textbf{0.0000} $\pm$ \textbf{0.0000} & 0.0231 $\pm$ \textbf{0.0231} & \textbf{1.000} $\pm$ \textbf{0.000} & \textbf{0.0003} $\pm$ \textbf{0.0000} \\
\textbf{AVCG-R ($\epsilon=0.0$)} & \textbf{1.000} $\pm$ \textbf{0.000} & \textbf{0.925} $\pm$ \textbf{0.207} & 6.278 $\pm$ 0.084 & 0.832 $\pm$ 0.117 & \textbf{1.000} $\pm$ \textbf{0.000} & \textbf{0.0000} $\pm$ \textbf{0.0000} & 0.0295 $\pm$ 0.0293 & \textbf{1.000} $\pm$ \textbf{0.000} & 0.0004 $\pm$ 0.0001 \\
\midrule
\multicolumn{10}{l}{\textbf{Environment: $\epsilon = 0.8$ (Relaxed Empirical Bound)}} \\
\midrule
CF-VAE        & 0.334 $\pm$ 0.472 & 3.216 $\pm$ 3.381 & 8.242 $\pm$ 2.946 & \textbf{2.454} $\pm$ \textbf{0.484} & 0.348 $\pm$ 0.432 & 0.0044 $\pm$ 0.0103 & 0.3346 $\pm$ 0.2369 & 0.333 $\pm$ 0.435 & 0.0011 $\pm$ 0.0001 \\
Wachter et al.& 0.716 $\pm$ 0.451 & 2.750 $\pm$ 2.440 & 9.535 $\pm$ 3.475 & 0.000 $\pm$ 0.000 & 0.495 $\pm$ 0.331 & 0.0221 $\pm$ 0.0179 & \textbf{0.0051} $\pm$ \textbf{0.0168} & 0.569 $\pm$ 0.100 & 0.0301 $\pm$ 0.0448 \\
\textbf{AVCG-B}        & \textbf{1.000} $\pm$ \textbf{0.000} & \textbf{0.983} $\pm$ \textbf{0.244} & \textbf{6.276} $\pm$ \textbf{0.091} & \textbf{0.776} $\pm$ \textbf{0.128} & \textbf{1.000} $\pm$ \textbf{0.000} & \textbf{0.0000} $\pm$ \textbf{0.0000} & 0.0242 $\pm$ 0.0249 & \textbf{1.000} $\pm$ \textbf{0.001} & \textbf{0.0003} $\pm$ \textbf{0.0000} \\
\textbf{AVCG-R ($\epsilon=0.8$)} & \textbf{1.000} $\pm$ \textbf{0.000} & 1.037 $\pm$ 0.238 & 6.332 $\pm$ 0.087 & 0.690 $\pm$ 0.108 & \textbf{1.000} $\pm$ \textbf{0.000} & \textbf{0.0000} $\pm$ \textbf{0.0000} & 0.0208 $\pm$ 0.0239 & \textbf{1.000} $\pm$ \textbf{0.000} & 0.0004 $\pm$ 0.0001 \\
\bottomrule
\end{tabular}
}
\end{table*}


\end{document}